\documentclass[letterpaper, 10 pt, conference]{ieeeconf} 
\usepackage[LGR,T1]{fontenc}
\usepackage[latin9]{inputenc}
\usepackage{verbatim}
\usepackage{amsmath}
\usepackage{amssymb}
\usepackage{graphicx}
\usepackage{graphics}
\usepackage{epsfig}
\usepackage{epstopdf}
\usepackage[caption=false]{subfig}
\usepackage{cite}

\usepackage{amsthm}	

\newtheorem{rem}{Remark}
\newtheorem{problem}{Problem}
\newtheorem{thm}{Theorem}
\newtheorem{defn}{Definition}
\newtheorem{prop}{Proposition}

\IEEEoverridecommandlockouts
\renewcommand{\baselinestretch}{1}  
\title{\LARGE \bf Distributed Cooperative Manipulation under Timed Temporal Specifications \rm}

\author{Christos K. Verginis and Dimos V. Dimarogonas \thanks{The authors are with the Centre for Autonomous Systems and ACCESS Linnaeus Centre, KTH Royal Institute of Technology, Stockholm 10044, Sweden. Emails: \{cverginis, dimos\}@kth.se.} \thanks{This work was supported by funding from the Knut and Alice Wallenberg Foundation, the Swedish Research Council (VR), the European Union's Horizon 2020 Research and Innovation Programme under the Grant Agreement No. 644128 (AEROWORKS) and the H2020 ERC Starting Grant BUCOPHSYS.}}

\begin{document}

\maketitle
\thispagestyle{empty}
\pagestyle{empty}

\begin{abstract}
This paper addresses the problem of cooperative manipulation of a single object by $N$ robotic agents under local goal specifications given as Metric Interval Temporal Logic (MITL) formulas. In particular, we propose a distributed model-free control protocol for the trajectory tracking of the cooperatively manipulated object without necessitating feedback of the contact forces/torques or inter-agent communication. This allows us to abstract the motion of the coupled object-agents system as a finite transition system and, by employing standard automata-based methodologies, we derive a hybrid control algorithm for the satisfaction of a given MITL formula. In addition, we use load sharing coefficients to represent potential differences in power capabilities among the agents. Finally, simulation studies verify the validity of the proposed scheme.

\end{abstract}

\section{Introduction \label{sec:Introduction}}

Multi-agent systems have gained significant attention over the last decade, since they provide several advantages with respect to single-agent setups. In the case of object manipulation, complex tasks involving heavy/large payloads and difficult maneuvers necessitate the employment of more than one robot. The problem of cooperative manipulation control has been studied extensively, using centralized schemes, where a central computer handles the agents' behavior, as well as decentralized setups, where each agent determines its actions on its own, either with partial or no communication at all \cite{Liu1998,Caccavale2008,Heck2013,Szewczyk2002,Tsiamis2015,Petitti_ICRA16,Wang_CDC16,sugar2002control,tanner2003nonholonomic,Erhart2013_2, MARKDAHL_IFAC2012}.   

In contrast to the related literature, which mainly considers the trajectory tracking of the manipulated object, we would like to define complex tasks over time, such as "never take the object to dangerous regions" or "keep moving the object from region A to B within a predefined time interval" which must be executed via the control actions of the agents. Such tasks can be expressed by temporal logic languages, which can describe complex planning objectives more efficiently than the well-studied navigation algorithms. Linear Temporal Logic (LTL) is the most common language that has been incorporated to the multi-agent motion planning problem \cite{guo2014cooperative,chen2012formal,Zhang_ACC16,Filippidis_ACC16,Mercado_CDC15}, without however considering time specifications. Metric and Metric Interval Temporal Logic (MTL, MITL) \cite{Alur94_timed_automata,Alur96,DSouza07} as well as Time Window Temporal Logic (TWTL) are languages that encode time specifications and were used for multi-agent motion planning in \cite{Alex16,Karaman2011,Aksaray_ICRA16}.

Regarding robotic manipulation, high level planning techniques have been proposed in \cite{Yamashita2003,kumar2009abstractions,Lionis2005} using common planning methods like configuration space potential fields and $A^{*}$ algorithms. In \cite{Lionis2005} the motion planning problem for a group of unicycles manipulating a rigid body is addressed and in \cite{kumar2009abstractions} an abstraction methodology is introduced; LTL specifications are employed in \cite{TsiamisJana2015}, where two mobile robots transport an object in a leader-follower scheme. Additionally, temporal logic formulas are utilized in \cite{Murray2012} for dexterous manipulation through robotic fingers and in \cite{Kavraki2015} for single manipulation tasks, without, however, incorporating the dynamics of the robotic arm in the abstracted model. In \cite{Chaimowicz2003} a hybrid framework for cooperative manipulation is presented.
 
For the continuous control part, impedance and/or force control is the most common methodology \cite{Caccavale2008,Heck2013,Szewczyk2002,Tsiamis2015}, in which the robotic arms employ sensors to obtain feedback of the contact forces/torques which, however, may result to performance decline due to sensor noise or mounting difficulties. Moreover, most works in the related literature consider known dynamic models and/or parameters of the object and the agents, whose accurate knowledge, however, can be a challenging issue. 

In this paper, we propose a novel hybrid control scheme for the cooperative manipulation of an object under MITL specifications.
In particular, we develop a distributed model-free control protocol for the trajectory tracking of the cooperative manipulated object with prescribed transient and steady state performance. The latter allows us to abstract the motion of the coupled system object-agents as a finite transition system. Then, by employing formal verification-based  methodologies, we derive a path that satisfies a given MITL task. The control scheme does not use any force/torque information at the contact points or any inter-agent communication and incorporates load sharing coefficients to account for differences in power capabilities among the agents.

The rest of the paper is organized as follows: Section \ref{sec:Notation-and-Preliminaries} introduces notation and preliminary background. Section \ref{sec:Problem-Formulation} describes the problem formulation and the overall system's model. The control strategy is presented in Section \ref{sec:main results}. Section \ref{sec:Simulation Results} verifies our approach through numerical simulation results and Section \ref{sec:Conclusion} concludes the paper.

\section{Notation and Preliminaries\label{sec:Notation-and-Preliminaries}}
\subsection{Notation}
The set of positive integers is denoted as $\mathbb{N}$ and the real $n$-coordinate space, with $n\in\mathbb{N}$, as $\mathbb{R}^n$;
$\mathbb{R}^n_{\geq 0}$ and $\mathbb{R}^n_{> 0}$ are the sets of real $n$-vectors with all elements nonnegative and positive, respectively. Given a set $S$, $2^S$ is the set of all subsets of $S$. The vector connecting the origins of coordinate frames $\{A\}$ and $\{B$\} expressed in frame $\{C\}$ coordinates in $3$-D space is denoted as $p^{\scriptscriptstyle C}_{{\scriptscriptstyle B/A}}\in{\mathbb{R}}^{3}$. Given $a\in\mathbb{R}^3$, $S(a)$ is the skew-symmetric matrix
defined according to $S(a)b = a\times b$. The rotation matrix from $\{A\}$ to $\{B\}$ is denoted as $R_{{\scriptscriptstyle B/A}}\in SO(3)$, where $SO(3)$ is the $3$-D rotation group.
The angular velocity of frame $\{B\}$ with respect to $\{A\}$ is
denoted as $\omega_{{\scriptscriptstyle B/A}}\in \mathbb{R}^{3}$ and it holds that \cite{Siciliano2009} $\dot{R}_{{\scriptscriptstyle B/A}}=S(\omega_{{\scriptscriptstyle B/A}})R_{{\scriptscriptstyle B/A}}$. We further denote as $\eta_{\scriptscriptstyle A/B}\in\mathbb{T}^3$ the Euler angles representing the orientation of $\left\{B\right\}$ with respect to $\left\{A \right\}$ and $\omega_{{\scriptscriptstyle B/A}} = J_{\scriptscriptstyle B/A}(\eta_{\scriptscriptstyle B/A})  \dot{\eta}_{\scriptscriptstyle B/A}$; $J_{\scriptscriptstyle B/A}:\mathbb{T}^3\rightarrow\mathbb{R}^{3\times3}$ is a smooth function representing the analytic Jacobian and $\mathbb{T}^3$ is the $3$-D torus. Moreover, $\mathcal{B}(c,r)$ denotes the $3$-D sphere of radius $r\geq 0$ and center $c\in\mathbb{R}^{3}$ and $d:\mathbb{R}^3\times\mathbb{R}^3\rightarrow\mathbb{R}$ is the $3$-D Euclidean distance. We further define the set $\mathbb{M} = \mathbb{R}^3\times\mathbb{T}^3$. For notational brevity, when a coordinate frame corresponds to an inertial frame of reference $\left\{I\right\}$, we will omit its explicit notation (e.g., $p_{\scriptscriptstyle B} = p^{\scriptscriptstyle I}_{\scriptscriptstyle B/I}, \omega_{\scriptscriptstyle B} = \omega^{\scriptscriptstyle I}_{\scriptscriptstyle B/I}, R_{\scriptscriptstyle B} = R_{\scriptscriptstyle B/I}$ etc.). Finally, all vector and matrix differentiations will be with respect to an inertial frame $\{I\}$, unless otherwise stated.

\subsection{Metric Interval Temporal Logic (MITL)}
The syntax of \textit{MITL} over a set of atomic propositions $\mathcal{AP}$ is defined by the grammar $\phi := p \:|\: \neg \phi \:|\: \bigcirc_I \phi \:|\: \lozenge_I \phi \:|\: \square_I \phi \:|\:\phi_{1}\mathcal{U}\phi_{2} $, where $p\in\mathcal{AP}$, $\bigcirc,\lozenge,\square$ and $\cup$ are the next, future, always and until operators, respectively. $I$ is one of the following intervals: $[i_1,i_2],[i_1,i_2),(i_1,i_2],$ $(i_1,i_2),[i_1,\infty),(i_1,\infty)$ with $i_1,i_2\in\mathbb{R}_{\geq 0}, i_2>i_1$.  

Given a set of atomic propositions $\mathcal{AP}$, an  MITL formula $\phi$ and an infinite sequence $r = (r_1,t_1)(r_2,t_2)\dots$, with $r_j\in 2^{AP}$ and $t_j\in\mathbb{R}_{\geq 0}, t_{j+1}>t_j, \forall j\in\mathbb{N}$, we define $(r,j)\models\phi, j\in\mathbb{N}$ ($r$ satisfies $\phi$ at $j$) in the point-wise semantics \cite{DSouza07}:
\small
\begin{eqnarray} 
(r,j) &\models& p \Leftrightarrow p\in r_j, \nonumber \\
(r,j) &\models& \neg\phi \Leftrightarrow (r,j) \not\models \phi \nonumber \\
(r,j) &\models&  \phi_1\land\phi_2 \Leftrightarrow (r,j) \models \phi_1 \text{ and } (r,j) \models \phi_2 \nonumber \\
(r,j) &\models&  \bigcirc_I\phi \Leftrightarrow (r,j+1)\models\phi \text{ and } t_{j+1}-t_j\in I \nonumber \\
(r,j) &\models&  \phi_1\mathcal{U}\phi_2 \Leftrightarrow \exists k,j, j\leq k, \text{s.t. } (r,k) \models \phi_2, \nonumber\\
	 && t_k - t_j \in I \text{ and } (r,m)\models \phi_1, \forall j\leq m \leq k.	\nonumber
\end{eqnarray}
\normalsize
Also, $\lozenge_I\phi = \text{True }\mathcal{U}\phi$ and $\square_I\phi = \neg\lozenge_I\neg\phi$. The sequence $r$ satisfies $\phi$, denoted as $r\models\phi$ if and only if $(r,1)\models\phi$. More information regarding MITL can be found in \cite{Alur96}, \cite{DSouza07}.

\subsection{Dynamical Systems}
Consider the initial value problem: 
\begin{equation}
\dot{\xi} = H(t,\xi), \xi(t_0)=\xi^0\in\Omega_\xi, \label{eq:initial value pr}
\end{equation}
with $H:[t_0,+\infty)\times\Omega_\xi\rightarrow\mathbb{R}^n$ where $\Omega_\xi\subset\mathbb{R}^n$ is a non-empty open set.
\begin{defn}
\cite{sontag2013mathematical} A solution $\xi(t)$ of the initial value problem (\ref{eq:initial value pr}) is maximal if it has no proper right extension that is also a solution of (\ref{eq:initial value pr}).
\end{defn}
\begin{thm} \label{Th:dynamical systems}
\cite{sontag2013mathematical} Consider problem (\ref{eq:initial value pr}). Assume that $H(t,\xi)$ is: a) locally Lipschitz on $\xi$ for almost all $t\in[t_0,+\infty)$, b) piecewise continuous on $t$ for each fixed $\xi\in\Omega_\xi$ and c) locally integrable on $t$ for each fixed $\xi\in\Omega_\xi$. Then, there exists a maximal solution $\xi(t)$ of (\ref{eq:initial value pr}) on $[t_0,t_{\max})$ with $t_{\max} > t_0$ such that $\xi(t)\in\Omega_{\xi},\forall t\in[t_0,t_{\max})$.
\end{thm}
\begin{prop} \label{Prop:dynamical systems}
\cite{sontag2013mathematical} Assume that the hypotheses of Theorem \ref{Th:dynamical systems} hold. For a maximal solution $\xi(t)$ on the time interval $[t_0,t_{\max})$ with $t_{\max} < \infty$ and for any compact set $\Omega'_\xi\subset \Omega_\xi$ there exists a time instant $t'\in[t_0,t_{\max})$ such that $\xi(t')\notin\Omega'_\xi$.
\end{prop}

\begin{figure}[!btp]
\centering
\includegraphics[trim = 0cm 0cm 0cm -1.2cm,width = 0.35\textwidth]{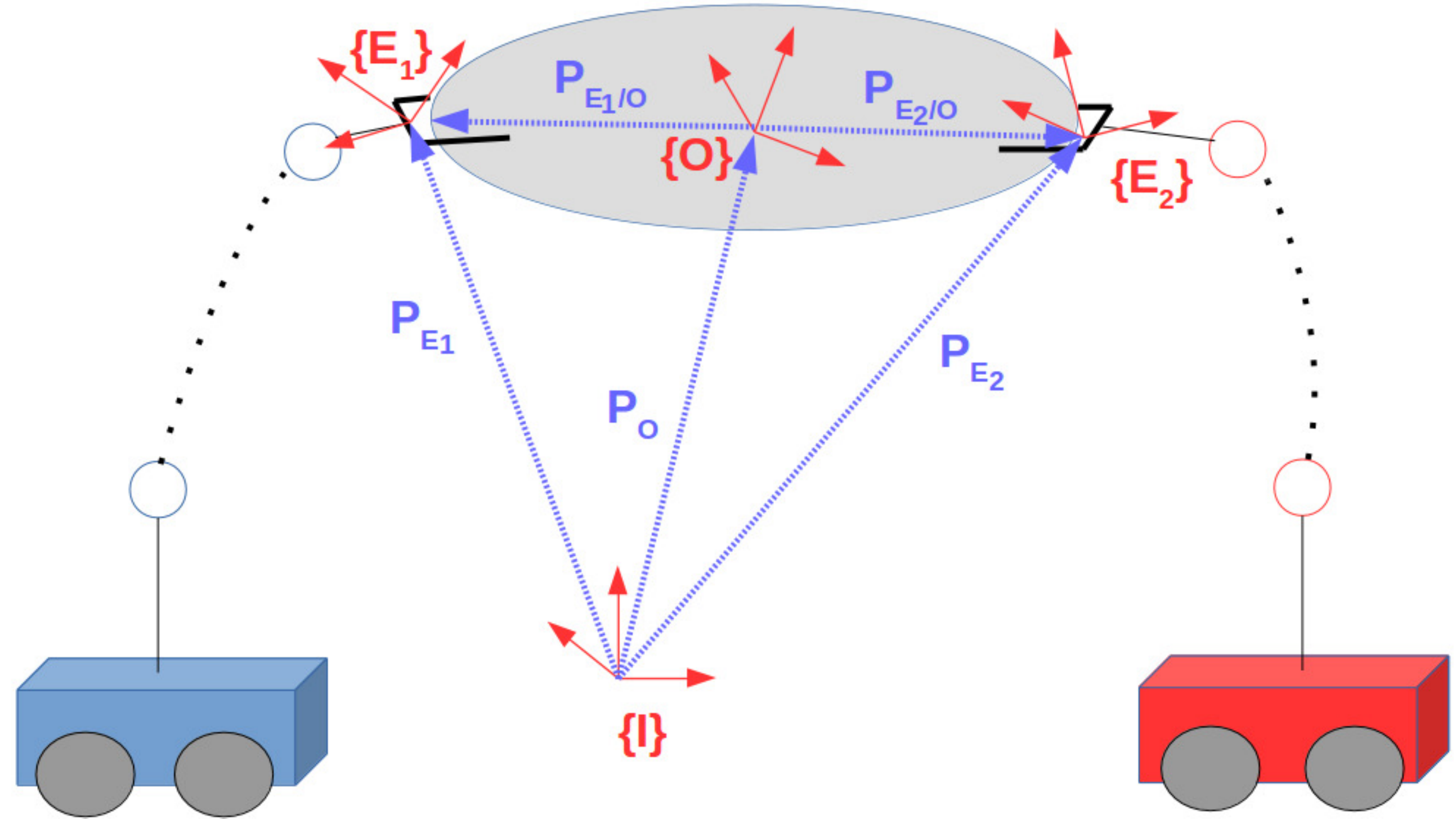}

\caption{Two robotic arms rigidly grasping an object.\label{fig:Two-robotic-arms}}
\end{figure}

\section{Problem Formulation\label{sec:Problem-Formulation}}

Consider a bounded workspace $\mathcal{W} \subset \mathbb{R}^{3}$ consisting of $N$ mobile manipulators rigidly grasping an object, as shown in Fig. \ref{fig:Two-robotic-arms}. We assume that each agent $i\in\{1,\dots,N\}$ has $\mathfrak{n}_i \geq 6$ degrees of freedom (generalized joint coordinates), denoted as $q_i:\mathbb{R}_{\geq 0}\rightarrow\mathbb{R}^{\mathfrak{n}_i}$. We also denote the entire joint space as $\overline{q}=[q^T_1,\cdots,q^T_N]^T:\mathbb{R}_{\geq 0}\rightarrow\mathbb{R}^{\mathfrak{n}}$, with $\mathfrak{n} =$ \small$\sum_{i=1}^{N}\mathfrak{n}_i$\normalsize. The reference frames corresponding to the $i$-th end-effector and the object's center of
mass are denoted with $\left\{ E_{i}\right\} $ and $\left\{ O\right\} $,
respectively, whereas $\left\{ I\right\} $ corresponds to an inertial
 reference frame. The rigidity of the grasps implies that the agents can exert any forces/torques along every direction to the object. We consider that each agent $i$ knows the position and velocity only of its own joint variables $q_i$ as well as its own and the object's geometric parameters. Moreover, no interaction force/torque measurements or on-line communication is required and the dynamic model of the object and the agents is considered unknown. Finally, we assume that the robotic agents and the object are away from kinematic and representation singularities \cite{Siciliano2009}. 

\subsection{System model \label{subsec:system_model}}

\subsubsection{Kinematics} 

In view of Fig. \ref{fig:Two-robotic-arms}, we have that 
\begin{eqnarray}
p_{{\scriptscriptstyle E_{i}}}(\overline{q}) & = &  p_{{\scriptscriptstyle O}}(\overline{q})+ p_{{\scriptscriptstyle E_{i}/O}}(\overline{q}) =p_{{\scriptscriptstyle O}}(\overline{q})+R_{{\scriptscriptstyle E_i}}(\overline{q}) p^{\scriptscriptstyle E_i}_{{\scriptscriptstyle E_{i}/O}} \nonumber \\ 
\eta_{\scriptscriptstyle E_i}(\overline{q})& =& \eta_{\scriptscriptstyle O}(\overline{q}) + \alpha_i, \label{eq: object agent pos}
\end{eqnarray}
where $\alpha_i\in\mathbb{T}^3$ is a known angular offset. We further define $x_{\scriptscriptstyle E_i}, x_{\scriptscriptstyle O}:\mathbb{R}^{\mathfrak{n}}\rightarrow\mathbb{M}$ with $x_{\scriptscriptstyle E_i}(\overline{q}) = [p^T_{{\scriptscriptstyle E_{i}}}(\overline{q}), \eta^T_{{\scriptscriptstyle E_{i}}}(\overline{q})]^T$ and $x_{\scriptscriptstyle O}(\overline{q}) = [p^T_{{\scriptscriptstyle O}}(\overline{q}), \eta^T_{{\scriptscriptstyle O}}(\overline{q})]^T$.  Note that, since (\ref{eq: object agent pos}) holds for all $i\in\{1,\dots,N\}$, the physical coupling between the object and the agents creates a dependence of $x_{\scriptscriptstyle O}$ and $x_{\scriptscriptstyle E_i}$ on all $\overline{q}$. Moreover, the grasp rigidity implies that $\omega_{{\scriptscriptstyle E_{i}}}=\omega_{{\scriptscriptstyle O}}$, i.e., $J_{\scriptscriptstyle E_i}(\eta_{\scriptscriptstyle E_i}(\overline{q}))\dot{\eta}_{\scriptscriptstyle E_i}=J_{\scriptscriptstyle O}(\eta_{\scriptscriptstyle O}(\overline{q}))\dot{\eta}_{\scriptscriptstyle O}$. Since the agents and the object are away from representation singularities, $J_{\scriptscriptstyle O}^{-1}$ and $J_{\scriptscriptstyle E_i}^{-1}$ are well-defined and smooth and hence we differentiate (\ref{eq: object agent pos}) to obtain:
\begin{equation}
\dot{x}_{{\scriptscriptstyle E_{i}}}=J_{{\scriptscriptstyle O_i}}(x_{{\scriptscriptstyle E_{i}}},x_{{\scriptscriptstyle O}})\dot{x}_{{\scriptscriptstyle O}},\label{eq:object-end-effector jacobian}
\end{equation}
where $J_{\scriptscriptstyle O_i}:\mathbb{M}\times\mathbb{M}\rightarrow\mathbb{R}^{6\times6}$ is the Jacobian from the object to the $i$-th agent:  
\begin{equation}
J_{{\scriptscriptstyle O_i}}=\left[\begin{array}{cc}
I_{{\scriptscriptstyle 3\times3}} & S(p_{{\scriptscriptstyle O/E_{i}}})J_{\scriptscriptstyle O}(\eta_{\scriptscriptstyle O})\\
0_{{\scriptscriptstyle 3\times3}} & J^{-1}_{\scriptscriptstyle E_{i}}(\eta_{\scriptscriptstyle E_{i}})J_{\scriptscriptstyle O}(\eta_{\scriptscriptstyle O})
\end{array}\right],
 \label{eq:jacobian O_i}
\end{equation}
which is always invertible due to the grasp rigidity. 

Furthermore, by noticing that $R_{\scriptscriptstyle E_i}(\overline{q})p^{\scriptscriptstyle E_i}_{\scriptscriptstyle E_i/O}=R_{\scriptscriptstyle O}(\eta_{\scriptscriptstyle O})p^{\scriptscriptstyle O}_{\scriptscriptstyle E_i/O}$, (\ref{eq: object agent pos}) can be rewritten as 
\begin{equation}
x_{\scriptscriptstyle E_i} = f_{\scriptscriptstyle O_i}(x_{\scriptscriptstyle O}), \label{eq:f_O}
\end{equation}
where $f_{\scriptscriptstyle O_i}:\mathbb{M}\rightarrow\mathbb{M}$ represents the coupled kinematics.  

\begin{rem}
Each agent $i$ can compute $x_{\scriptscriptstyle E_i},\dot{x}_{\scriptscriptstyle E_i}$ via its forward and differential kinematics \cite{Siciliano2009} $x_{\scriptscriptstyle E_i} = k_i(q_i)$ and $\dot{x}_{\scriptscriptstyle E_i} = J_i(q_i)\dot{q}_i$, respectively, where $k_i:\mathbb{R}^{\mathfrak{n}_i}\rightarrow\mathbb{R}^3$ and $J_i(q_i) = \partial k_i(q_i)/\partial q_i$ is the corresponding Jacobian. In addition, since the geometric parameters $p^{\scriptscriptstyle E_i}_{\scriptscriptstyle E_i/O}$ and $\alpha_i$ are known, $x_{\scriptscriptstyle O}$ and $\dot{x}_{\scriptscriptstyle O}$ can be computed by inverting eq. (\ref{eq: object agent pos}) and (\ref{eq:object-end-effector jacobian}), without employing any sensory data.
\end{rem}

\subsubsection{\bf Object Dynamics \rm}

The Newton-Euler equation for the object's second order dynamics is:
\small
\begin{eqnarray}
 M_{\scriptscriptstyle O}(x_{\scriptscriptstyle O})\ddot{x}_{{\scriptscriptstyle O}}+C_{{\scriptscriptstyle O}}(x_{\scriptscriptstyle O},\dot{x}_{\scriptscriptstyle O})\dot{x}_{{\scriptscriptstyle O}}+g_{\scriptscriptstyle O}(x_{\scriptscriptstyle O})+w_{\scriptscriptstyle O}(t)  =  \lambda_{\scriptscriptstyle O},\label{eq:object dynamics} 
\end{eqnarray}
\normalsize
where $M_{\scriptscriptstyle O}:\mathbb{M}\rightarrow\mathbb{R}^{6\times6}$ is the positive definite inertia matrix, $C_{{\scriptscriptstyle O}}:\mathbb{M}\times\mathbb{R}^6\rightarrow\mathbb{R}^{6\times6}$ is the Coriolis matrix, $g_{\scriptscriptstyle O}:\mathbb{M}\rightarrow\mathbb{R}^{6}$ is the gravity vector, $w_{\scriptscriptstyle O}:\mathbb{R}_{\geq 0}\rightarrow\mathbb{R}^6$ is a bounded vector representing external disturbances and $\lambda_{\scriptscriptstyle O}$ is the 	force vector acting on the object's center of mass. All aforementioned vector fields are continuous and unknown.

\subsubsection{\bf Agent Dynamics \rm}

The task-space dynamics for agent $i\in\{1,\dots,N\}$ are given by \cite{Siciliano2009}:
\begin{eqnarray}
M_{i}(x_{\scriptscriptstyle E_i})\ddot{x}_{{\scriptscriptstyle E_{i}}}+C_{i}(x_{\scriptscriptstyle E_i},\dot{x}_{\scriptscriptstyle E_i})\dot{x}_{{\scriptscriptstyle E_{i}}}+g_{i}(x_{\scriptscriptstyle E_i}) \nonumber \\ 
+ f_i(x_{\scriptscriptstyle E_i},\dot{x}_{\scriptscriptstyle E_i}) + w_i(t) + \lambda_{i}& =  & u_{i}   \label{eq:manipulator task space dynamics}
\end{eqnarray}
where $M_{i}:\mathbb{M}\rightarrow\mathbb{R}^{6\times6}$ is the task-space positive definite  inertia matrix,  $C_{i}:\mathbb{M}\times\mathbb{R}^6\rightarrow\mathbb{R}^6$ represents the task-space Coriolis matrix, $g_{i}:\mathbb{M}\rightarrow\mathbb{R}^6$
is the task-space gravity vector, $f_{i}:\mathbb{M}\times\mathbb{R}^6\rightarrow\mathbb{R}^6$ is a vector field representing model uncertainties and $w_i:\mathbb{R}_{\geq 0}\rightarrow\mathbb{R}^6$ is a bounded vector representing external disturbances. Similarly to (\ref{eq:object dynamics}), the aforementioned vector fields are continuous and completely unknown; $u_i\in\mathbb{R}^6$ is the task space wrench acting as the control input and  $\lambda_{i}\in\mathbb{R}^{6}$ is the generalized force vector that agent $i$ 
exerts on the object.
\begin{rem}
The task-space wrench $u_i$ can be translated to the joint space inputs $\tau_i\in\mathbb{R}^{\mathfrak{n}_i}$ via $\tau_{i}=J_{i}^{T}(q_i)u_{i}+(I_{{ \mathfrak{n_{i}}\times \mathfrak{n_{i}}}}-J^{T}_{i}(q_i)\bar{J}^{T}_{i}(q_i))\tau_{i_{0}}$, where  $\bar{J}_i$ is a generalized inverse of $J_i$ \cite{Siciliano2009}. The term $\tau_{i_{0}}$ concerns redundant agents ($\mathfrak{n}_i > 6$) and does not contribute to end-effector forces. 
\end{rem}

\subsubsection{\bf Coupled Dynamics \rm}

The kineto-statics duality \cite{Siciliano2009} along with the grasp rigidity suggest that the 
force $\lambda_{\scriptscriptstyle O}$ acting on the object center of mass and the generalized forces $\lambda_i, i\in\{1,\dots,N\}$ exerted by the
agents at the contact points are related through 
\begin{equation}
\lambda_{\scriptscriptstyle O}=G^{T}(\overline{x})\overline{\lambda},\label{eq:grasp matrix}
\end{equation}
where $\overline{x}=[x_{{\scriptscriptstyle O}}^{T}, \overline{x}_{{\scriptscriptstyle E}}^{T}]^T\in\mathbb{M}^{N+1}, \overline{x}_{{\scriptscriptstyle E}} = [x_{{\scriptscriptstyle E_{1}}}^{T},  \cdots, x_{{\scriptscriptstyle E_{N}}}^{T}]^{T}\in\mathbb{M}^{N}, \overline{\lambda}=[
\lambda^{T}_{1}, \cdots, \lambda^{T}_{N}]^{T}\in\mathbb{R}^{6N}$ and $G:\mathbb{M}^{N+1}\rightarrow{R}^{6N\times6}$ is the grasp matrix, with $G(\overline{x})=[
J_{{\scriptscriptstyle O_1}}^{T},\cdots,J_{{\scriptscriptstyle O_N}}^{T}]^{T}$. 

Next, we substitute (\ref{eq:object-end-effector jacobian}) and its derivative in (\ref{eq:manipulator task space dynamics}) and we obtain in vector form after rearranging terms:
\begin{eqnarray}
\overline{\lambda} &=& \overline{u} - \overline{M}(\overline{x}_{{\scriptscriptstyle E}})G(\overline{x})\ddot{x}_{\scriptscriptstyle O}  - \overline{g}(\overline{x}_{{\scriptscriptstyle E}})-\overline{f}(\overline{x}_{{\scriptscriptstyle E}}) - \overline{w}(t) \nonumber \\
& & - (\overline{M}(\overline{x}_{{\scriptscriptstyle E}})\dot{G}(\overline{x},\dot{\overline{x}}) + \overline{C}(\overline{x}_{{\scriptscriptstyle E}},\dot{\overline{x}}_{{\scriptscriptstyle E}})G(\overline{x})  ) \dot{x}_{\scriptscriptstyle O} \label{eq:coupled dynamics 1}
\end{eqnarray}
where we have used the stack forms $\overline{M}=  \text{diag}\{\left[M_{i}\right]_{i\in\{1,\dots,N\}}\},\overline{C} = \text{diag}\{\left[C_{i}\right]_{i\in\{1,\dots,N\}}\}, \overline{g}=[g_{1}^{T}, \cdots, g_{N}^{T}]^{T}, \overline{f}=[f_{1}^{T}, \cdots, f_{N}^{T}]^{T}, \overline{u}=[u_{1}^{T}, \cdots, u_{N}^{T}]^{T}$ and $\overline{w}=[w_{1}^{T}, \cdots, w_{N}^{T}]^{T}$.
By substituting (\ref{eq:coupled dynamics 1}) and (\ref{eq:object dynamics}) in (\ref{eq:grasp matrix}), we obtain the coupled dynamics:
\begin{equation}
\widetilde{M}(\overline{x})\ddot{x}_{{\scriptscriptstyle O}}+\widetilde{C}(\overline{x},\dot{\overline{x}})\dot{x}_{{\scriptscriptstyle O}}+\widetilde{h}(\overline{x},\dot{\overline{x}}) + \widetilde{w}(\overline{x},t) =G^{T}(\overline{x})\overline{u},\label{eq:coupled dynamics 2}
\end{equation}
where $\widetilde{M}(\overline{x}) =  M_{\scriptscriptstyle O}+G^{T}\overline{M}G, \widetilde{C}(\overline{x},\dot{\overline{x}})  =  C_{\scriptscriptstyle O}+G^{T}\overline{M}\dot{G}+G^{T}\overline{C}G, \widetilde{h}(\overline{x}, \dot{\overline{x}}) = g_{\scriptscriptstyle O} + G^{T}\overline{g}+G^{T}\overline{f}$, and $\widetilde{w}(\overline{x},t) = w_{\scriptscriptstyle O} + G^{T}\overline{w}$.
It is straightforward to deduce the positive definiteness of $\widetilde{M}$ as well as the continuity of all the above coupled vector fields due to the continuity of the individual terms.

\begin{figure}[!tbp]
  \centering
  \subfloat[Illustration of $\hat{L}$]{\includegraphics[trim = 0cm 0.5cm 1.5cm -0.75cm,width=0.3\textwidth]{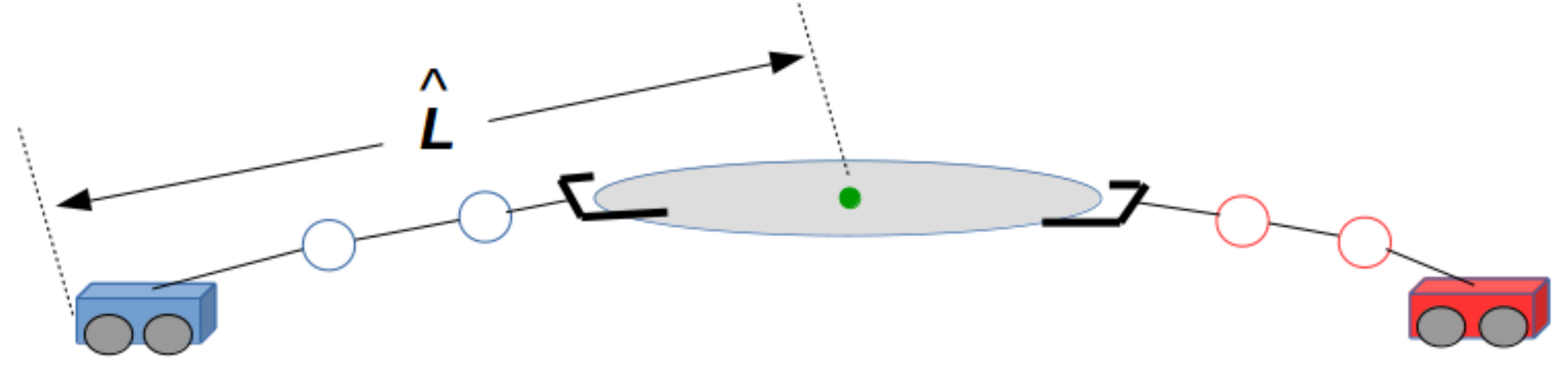}\label{fig:worksp_discretization:L0}}
  \hfill
  \subfloat[Workspace Discretization]{\includegraphics[trim = 0cm -0.5cm 0cm -1.5cm,width=0.3\textwidth]{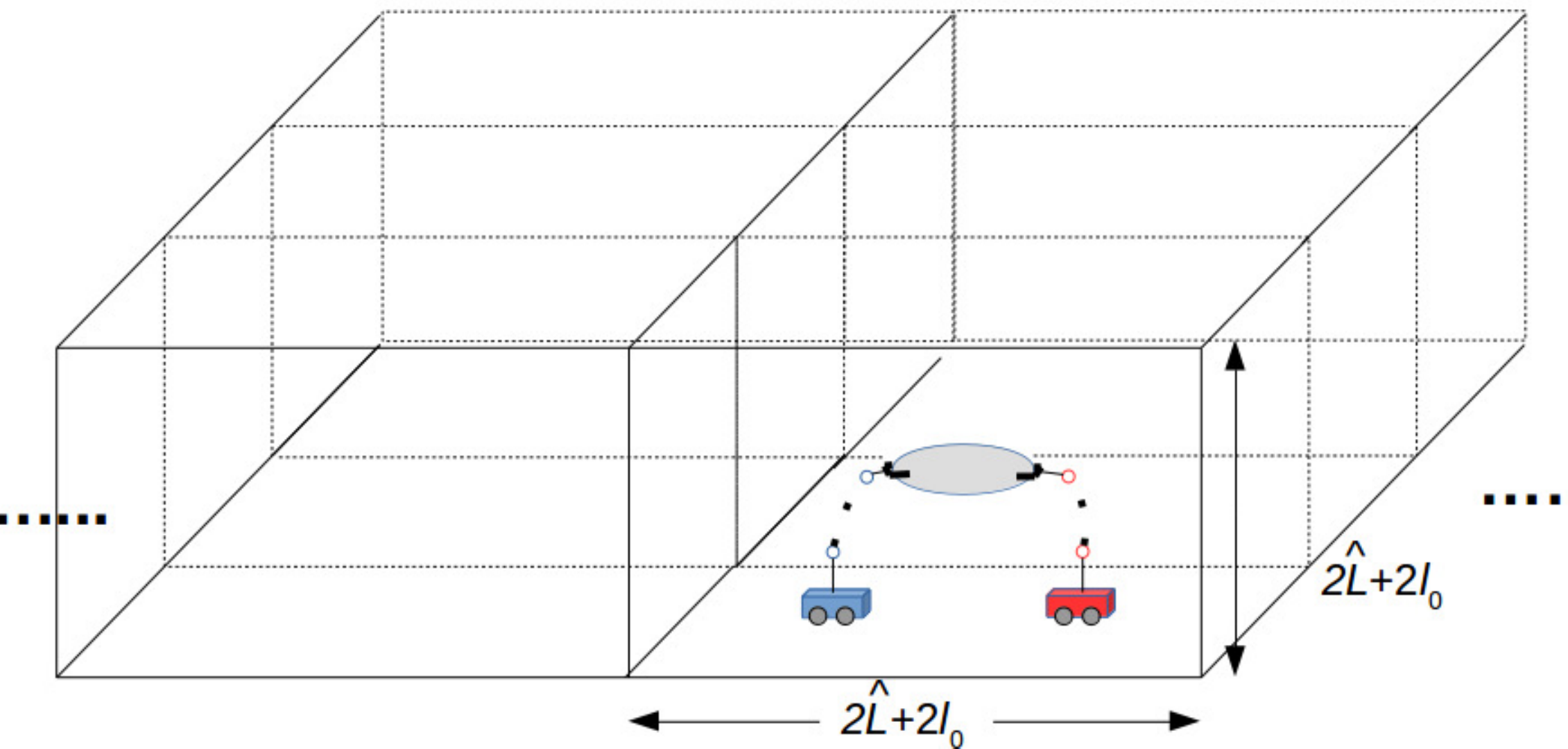}\label{fig:worksp_discretization:grid3d}}
  \caption{(a): An example of the agents of Fig. \ref{fig:Two-robotic-arms} in the configuration that derives $\hat{L}$. (b): The workspace partition according to the bounding box of the coupled system.}\label{fig:worksp_discretization}
\end{figure}

\subsection{Workspace Partition} \label{subsec:wsp discret}

As mentioned in Section \ref{sec:Introduction}, we are interested in defining MITL formulas over certain properties in a discrete set of regions of the workspace. Therefore, we provide now a partition of $\mathcal{W}$ into cell regions. We first define the set $\mathcal{S}_{\overline{q}}$ that consists of all points $p_s\in\mathcal{W}$ that physically belong to the coupled system, i.e. they consist part of either the agents' or the object' volume. Note that these points depend directly on the actual value of $\overline{q}$. We further define the constant $\hat{L} \geq \sup_{\substack{\overline{q}\in\mathbb{R}^{\mathfrak{n}}\\p_s\in\mathcal{S}_{\overline{q}}}} d(p_s,p_{\scriptscriptstyle O}(\overline{q}))$. Note that, although the explicit computation of $\mathcal{S}_{\overline{q}}$ may not be possible, $\hat{L}$ is an upper bound of the maximum distance between the object center of mass and a point in the coupled system's volume over all possible configurations $\overline{q}$, and thus, it can be measured. For instance, Fig. \ref{fig:worksp_discretization}\subref{fig:worksp_discretization:L0} shows $\hat{L}$ for the system of Fig. \ref{fig:Two-robotic-arms}. It is straightforward to conclude that
\begin{equation}
p_s\in\mathcal{B}(p_{\scriptscriptstyle O}(\overline{q}),\hat{L}),\forall p_s\in\mathcal{S}_{\overline{q}}, \overline{q}\in\mathbb{R}^{\mathfrak{n}}. \label{eq:p_s in L_hat}
\end{equation}
 Next, we partition the workspace $\mathcal{W}$ into $R$ equally sized rectangular regions $\Pi=\left\{\pi_{1},\dots,\pi_{R}\right\}$, whose geometric centers are denoted as $p^c_{\pi_j}\in\mathcal{W}, j\in\{1,\dots,R\}$. The length of the region sides is set as $D = 2\hat{L}+2l_0$, where $l_0$ is an arbitrary positive constant. Hence, each region $\pi_j$ can be formally defined as follows: 
 \small
 \begin{eqnarray}
\pi_j  &=& \{ p\in\mathcal{W} \text{ s.t. } (p)_k\in [ (p^c_{\pi_j})_k - \hat{L}-l_0, (p^c_{\pi_j})_k + \hat{L}+l_0 ), \nonumber \\
   & &   \forall k\in\{x,y,z\} \},  \label{eq:discretization 1} \nonumber
\end{eqnarray}
\normalsize
with $d(p^c_{\pi_{j+1}},p^c_{\pi_{j}}) = (2\hat{L} + 2l_0), \forall j\in\{1,\dots,R-1\}$ and $(p^c_{\pi_j})_z = \hat{L} + l_0, \forall j\in\{1,\dots,R\}$. The notation $(a)_k, k\in\{x,y,z\}$ denotes the $k$-th coordinate of $a=[(a)_x,(a)_y,(a)_z]^T\in\mathbb{R}^3$. An illustration of the aforementioned partition is depicted in \ref{fig:worksp_discretization}\subref{fig:worksp_discretization:grid3d}.

Note that each $\pi_j$ is a uniformly bounded, convex and well-connected set and also $\pi_j \cap \pi_{j'} = \emptyset, \forall j,j'\in\{1,\dots,R\}$ with $j\neq j'$. We also define the neighborhood of region $\pi_j$ as the set of its adjacent regions $\mathcal{D}(\pi_j) = \{\pi_{j'}\in\Pi \text{ s.t. } d(p^c_{\pi_j},p^c_{\pi_{j'}}) = (2\hat{L} + 2l_0) \}$, which is symmetric, i.e., $\pi_{j'}\in\mathcal{D}(\pi_j) \Leftrightarrow \pi_{j}\in\mathcal{D}(\pi_{j'})$.

To proceed we need the following definitions regarding the timed transition of the coupled system between two regions $\pi_j,\pi_{j'}$: 
\begin{defn} \label{def:system in region}
The coupled system object-agents is in region $\pi_j$ at a configuration $\overline{q}$, denoted as $\mathcal{A}(\overline{q})\in\pi_j$, if and only if (i) $p_s\in\pi_j, \forall p_s\in\mathcal{S}_{\overline{q}}$ and (ii) $d(p_{\scriptscriptstyle O}(\overline{q}),p^c_{\pi_j}) < l_0$.
\end{defn}

\begin{defn} \label{def:transition}
Assume that $\mathcal{A}(\overline{q}(t_0))\in\pi_j, j\in\{1,\dots,R\}$, for some $t_0\in\mathbb{R}_{\geq 0}$. Then, there exists a transition for the coupled system object-agents from $\pi_j$ to $\pi_{j'}, j'\in\{1,\dots,R\}$ with time duration $\delta t_{j,j'}\in\mathbb{R}_{\geq 0}$, denoted as $\pi_j\rightarrow\pi_{j'}$, if and only if there exists a bounded control trajectory $\overline{u}$ in (\ref{eq:coupled dynamics 2}), such that $\mathcal{A}(\overline{q}(t_0 + \delta t_{j,j'}))\in\pi_{j'}$ and $p_s\in\pi_j\cup\pi_{j'}, \forall p_s\in\mathcal{S}_{\overline{q}}, t\in[t_0,t_0 + \delta t_{j,j'}]$.
\end{defn}
Note that the entire system object-agents must remain in $\pi_j,\pi_{j'}$ during the transition and therefore the requirement $\pi_{j'}\in\mathcal{D}(\pi_j)$ is implicit in the definition above. 

\subsection{Specification} \label{subsec:specf}
Given the workspace partition, we can introduce a set of atomic propositions $\mathcal{AP}$ for the object, which are expressed as boolean variables that correspond to  properties of interest in the regions of the workspace (e.g., "Obstacle region", "Goal region"). Formally, the labeling function $\mathcal{L}:\Pi\rightarrow2^{\mathcal{AP}}$ assigns to each region $\pi_j$ the subset of the atomic propositions $\mathcal{AP}$ that are true in $\pi_j$. 
\begin{defn} \label{def:specification}
Assume that $\overline{q}(t)$ is a trajectory. Then, a \textit{timed sequence} of $\overline{q}$ is the infinite sequence $\beta=(\overline{q}_1(t),t_1,)(\overline{q}_2(t),t_2,)\dots$, with $t_m\in\mathbb{R}_{\geq 0}, t_{m+1} > t_m$ and $\mathcal{A}(\overline{q}_m(t_m))\in\pi_{j_m}, j_m\in\{1,\dots,R\}, \forall m\in\mathbb{N}$.
The \textit{timed behavior} of $\beta$ is the infinite sequence $\sigma_\beta = (\sigma_1,t_1)(\sigma_2,t_2)\dots$, with $\sigma_m\in2^{\mathcal{AP}}, \sigma_m\in\mathcal{L}(\pi_{j_m})$ for  $\mathcal{A}(\overline{q}_m(t_m))\in\pi_{j_m}, j_m\in\{1,\dots,R\}, \forall m\in\mathbb{N}$, i.e., the set of atomic propositions that are true when $\mathcal{A}(\overline{q}_m(t_m))\in\pi_{j_m}$. 
\end{defn}
\begin{defn}
The timed run $\beta$ satisfies an MITL formula $\phi$ if and only if $\sigma_{\beta} \models \phi$. 
\end{defn}

We are now ready to state the problem treated in this paper.

\begin{problem} \label{problem 1}
Given $N$ agents rigidly grasping an object in $\mathcal{W}$ subject to the coupled dynamics described by (\ref{eq:coupled dynamics 2}), the workspace partition $\Pi$ such that $\mathcal{A}(\overline{q}(0))\in\pi_j,j\in\{1,\dots,R\}$, a MITL formula $\phi$ over $AP$ and the labeling function $\mathcal{L}$, develop a control strategy that achieves a timed sequence $\beta$ which yields the satisfaction of $\phi$.   
\end{problem}

\section{Main Results} \label{sec:main results}

\begin{figure}[!btp]
\centering
\includegraphics[trim = 0cm 0cm 0cm -1cm,scale=0.25]{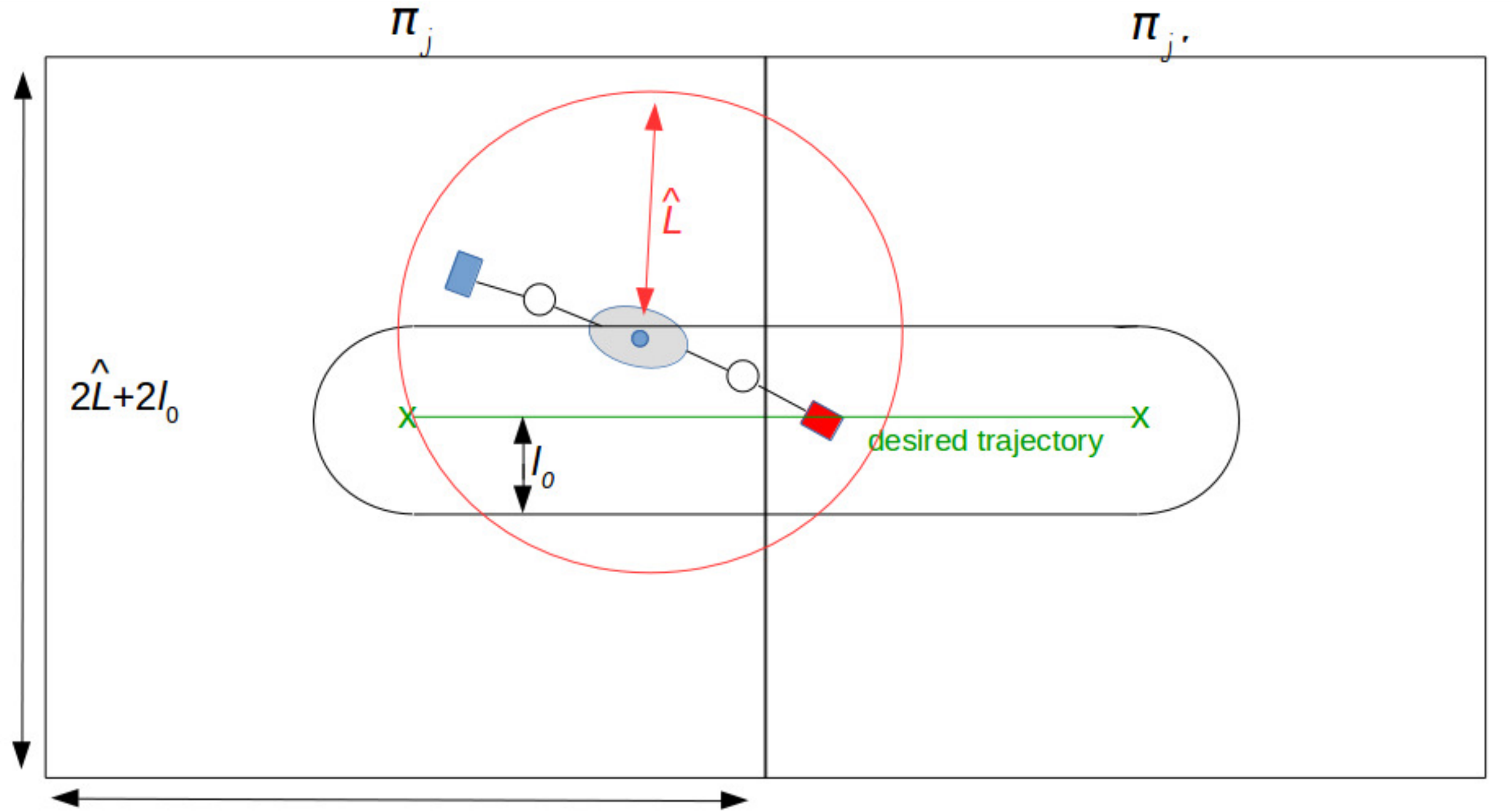}

\caption{A transition between two adjacent regions $\mathsf{\pi}_j$ and $\mathsf{\pi}_{j'}$. Since $p_{\scriptscriptstyle O}\in\mathcal{B}(p_d(t),l_0)$, we conclude that $p_s\in\mathcal{B}(p_{\scriptscriptstyle O},\hat{L})\in\mathcal{B}(p_d(t),l_0+\hat{L})\in\pi_j\cup\pi_{j'}$ . \label{fig:grid3}}
\end{figure}

\subsection{Control Design} \label{subsec:control design}
The first ingredient of the proposed solution, which constitutes one of the main novelties of this work, is the design of a decentralized control protocol $\overline{u}$ such that a transition relation between two adjacent regions according to Def. \ref{def:transition} is established. Assume, therefore, that $\mathcal{A}(\overline{q}(t_0))\in\pi_j, t_0\in\mathbb{R}_{\geq 0}$. We wish to find a bounded $\overline{u}$, such that (i) $\mathcal{A}(\overline{q}(t_0+\delta t_{j,j'}))\in\pi_{j'}$ with $\pi_{j'}\in\mathcal{D}(\pi_j)$ and (ii) $p_s\in\pi_j\cup\pi_{j'}, \forall p_s\in\mathcal{S}_{\overline{q}}, t\in[t_0,t_0+\delta t_{j,j'}]$, for a predefined arbitrary constant $\delta t_{j,j'}\in\mathbb{R}_{\geq 0}
$ corresponding to the transition $\pi_j \rightarrow \pi_{j'}$.

The first step is to associate to the transition a smooth trajectory defined by the line segment that connects $p^c_{\pi_j}$ and $p^c_{\pi_{j'}}$, i.e. define $p_d:[t_0,\infty)\rightarrow\mathbb{R}^3$, such that $p_d(t_0) =  p^c_{\pi_j}, p_d(t) = p^c_{\pi_{j'}}, \forall t\geq t_0+\delta t_{j,j'}$ and
\begin{equation}
\mathcal{B}(p_d(t), \hat{L}+l_0)\in\pi_j\cup\pi_{j'},\forall t\geq t_0.  \label{eq:desired_tr}
\end{equation}
If we guarantee that the object's center of mass stays $l_0$-close to $p_d$, i.e., $d(p_{\scriptscriptstyle O}(t),p_d(t))<l_0$, then $d(p_{\scriptscriptstyle O}(t_0+\delta t_{j,j'}), p^c_{\pi_{j'}}) < l_0$ and, by invoking (\ref{eq:p_s in L_hat}) and (\ref{eq:desired_tr}), we obtain $p_s\in\mathcal{B}(p_{\scriptscriptstyle O}(t),\hat{L})\in\mathcal{B}(p_d(t),\hat{L}+l_0)\in\pi_j\cup\pi_{j'},\forall p_s\in\mathcal{S}_{\overline{q}}, t\geq t_0$ (and therefore $t\in[t_0,t_0+\delta t_{j,j'}]$), and thus the requirements of Def. \ref{def:transition} for the transition relation are met. Fig. \ref{fig:grid3} illustrates the aforementioned reasoning.
Along with $p_d$, we also define a desired smooth trajectory $\eta_d:[t_0,\infty)\rightarrow\mathbb{T}^3$ for the object orientation and we form the desired pose trajectory $x_d:[t_0,\infty)\rightarrow\mathbb{M}$, with $x_d(t) = [x_{d_1}(t),\dots,x_{d_6}(t)]^T = [p^T_d(t), \eta^T_d(t)]^T$. We can now define the associated pose errors $e_s:[t_0,\infty)\rightarrow\mathbb{M}$ with 
\begin{equation}
e_s(t) = [e_{s_1}(t),\cdots, e_{s_6}(t)]^T = x_{\scriptscriptstyle O}(\overline{q}(t)) - x_d(t). \label{eq:e_s}
\end{equation}
A suitable methodology for the control design in hand is that of prescribed performance control, recently proposed in \cite{Bechlioulis2014}, which is adapted in this work in order to achieve predefined transient and steady state response bounds for the pose errors. As stated in \cite{Bechlioulis2014}, prescribed performance characterizes the behavior where the aforementioned errors evolve strictly within a predefined region that is bounded by absolutely decaying functions of time, called performance functions. The mathematical expressions of prescribed performance are given by the inequalities: $-\rho_{s_k}(t)<e_{s_k}(t)<\rho_{s_k}(t), \forall k\in\{1,\dots,6\}$, where $\rho_{s_k}:[t_0,\infty)\rightarrow\mathbb{R}_{> 0}$ with $\rho_{s_k}(t)=(\rho^{\scriptscriptstyle 0}_{s_k}-\rho^{\scriptscriptstyle \infty}_{s_k})e^{-l_{s_k}(t-t_0)}+\rho^{\scriptscriptstyle \infty}_{s_k}$ and $l_{s_k}>0, \rho^{\scriptscriptstyle 0}_{s_k}>\rho^{\scriptscriptstyle \infty}_{s_k}>0$,
are designer-specified, smooth, bounded and decreasing positive functions of time with positive parameters $l_{s_k},\rho^{\scriptscriptstyle \infty}_{s_k}$, incorporating the desired transient and steady state performance respectively. In particular, the decreasing rate of $\rho_{s_k}$, which is affected by the constant $l_{s_k}$, introduces a lower bound on the speed of convergence of $e_{s_k}(t)$. Furthermore, the constants $\rho^{\scriptscriptstyle \infty}_{s_k}$ can be set arbitrarily small, achieving thus practical convergence of the pose errors to zero.

Next, we propose a state 
feedback control protocol that does not incorporate any information on the agents' or the object's dynamics or the external disturbances and guarantees that $\left|e_{s_k}(t)\right|<\rho_{s_k}(t)$ for all $[t_0,\infty)$ and hence $[t_0,t_0+\delta t_{j,j'}]$, which,
by appropriately selecting $\rho_{s_k}(t)$ and given that $\mathcal{A}(\overline{q}(t_0))\in\pi_j$, guarantees a transition $\pi_j\rightarrow\pi_{j'}$ with time duration of $\delta t_{j,j'}$. Thus, given the pose errors (\ref{eq:e_s}): 

\textbf{Step I-a}. Select the corresponding functions $\rho_{s_k}(t)=(\rho^{\scriptscriptstyle 0}_{s_k}-\rho^{\scriptscriptstyle \infty}_{s_k})e^{-l_{s_k})(t-t_0)}+\rho^{\scriptscriptstyle \infty}_{s_k}$ with $\rho^{\scriptscriptstyle 0}_{s_k} = l_0, \forall k\in\{1,2,3\}, \rho^{\scriptscriptstyle 0}_{s_k} > \lvert e_{s_k}(t_0) \rvert, \forall k\in\{4,5,6\}$ and $\rho^{\scriptscriptstyle 0}_{s_k} > \rho^{\scriptscriptstyle \infty}_{s_k} >0 , l_{s_k}>0, \forall k\in\{1,\dots,6\}$.  

\textbf{Step I-b}. Define the normalized errors $\xi_{s_k}:[t_0,\infty)\rightarrow\mathbb{R}$:
\begin{equation}
\xi_{s_k}(t) = \rho_{s_k}^{-1}(t)e_{s_k}(t), \forall k\in\{1,\dots,6\},	\label{eq:ksi_s}
\end{equation}
and design the reference velocity vector as $\dot{x}^*_d:\mathbb{R}^6\times[t_0,\infty)\rightarrow\mathbb{R}^6$, with $\dot{x}^*_d = [\dot{x}^*_{d_1},\dots,\dot{x}^*_{d_6}]^T$ and: 
\begin{equation}
\dot{x}^*_{d_k}(\xi_{s_k},t) = -g_{s_k}\varepsilon_{s_k}(\xi_{s_k}), \forall k\in\{1,\dots,6\},	\label{eq:x_dot_d}
\end{equation}
where $g_{s_k} > 0$ and $\varepsilon_{s_k}:\mathbb{R}\rightarrow\mathbb{R}$ is the transformed error:
\begin{equation}
\varepsilon_{s_k}(\xi_{s_k}) =  \ln\left(\dfrac{1+\xi_{s_k}}{1-\xi_{s_k}}\right), \forall k\in\{1,\dots,6\}. \label{eq:epsilon_s}
\end{equation}
Before proceeding, note that (\ref{eq:e_s}) and (\ref{eq:ksi_s}) suggest that $x_{\scriptscriptstyle O} = P_s(t)\xi_s + x_d(t)$, where \small $\xi_s = [\xi_{s_1},\dots,\xi_{s_6}]^T$ \normalsize and \small$P_s = \text{diag}\{[\rho_{s_k}]_{k\in\{1,\dots,6\}}\}$ \normalsize. Therefore, by employing (\ref{eq:f_O}) we can write $\overline{x} = $ \small$[x^T_{\scriptscriptstyle O}, f^T_{\scriptscriptstyle O_1}(x_{\scriptscriptstyle O}),\dots,f^T_{\scriptscriptstyle O_N}(x_{\scriptscriptstyle O})]^T = [\xi^T_sP_s(t) + x^T_d(t), f^T_{\scriptscriptstyle O_1}(P_s(t)\xi_s + x_d(t)),\dots,f^T_{\scriptscriptstyle O_N}(P_s(t)\xi_s + x_d(t))]^T = \overline{f}_{\scriptscriptstyle O}(P_s(t)\xi_s + x_d(t))$\normalsize, i.e., we can express $\overline{x}$ as a function of $\xi_s$ and $t$. Therefore, in the following, the dependence on $\overline{x}$ will directly imply dependence on $\xi_s, t$. Similarly, dependence on $\dot{\overline{x}}$ will imply dependence on $\xi_s, \dot{\xi}_s$ and $t$. 

\textbf{Step II-a}. Define the velocity error vector $e_v:\mathbb{R}^6\times[t_0,\infty)\rightarrow\mathbb{R}^6$ with $e_v(\xi_s,t)=[e_{v_1}(\xi_{s_1},t),\dots,e_{v_6}(\xi_{s_6},t)]^T=\dot{x}_{\scriptscriptstyle O}(\overline{q}(t)) - \dot{x}^*_d(\xi_s,t)$ and select the corresponding positive performance functions $\rho_{v_k}:[t_0,\infty)\rightarrow\mathbb{R}_{>0}$ with $\rho_{v_k}(t) = (\rho^{\scriptscriptstyle 0}_{v_k} - \rho^{\scriptscriptstyle \infty}_{v_k})e^{-l_{v_k}(t-t_0)} + \rho^{\scriptscriptstyle \infty}_{v_k}$, such that $\rho^{\scriptscriptstyle 0}_{v_k} > \lvert e_{v_k}(t_0) \rvert, l_{v_k}>0$ and $\rho^{\scriptscriptstyle 0}_{v_k} > \rho^{\scriptscriptstyle \infty}_{v_k} >0,\forall k\in\{1,\dots,6\}$.

\textbf{Step II-b}. Define the normalized velocity errors $\xi_v:[t_0,\infty)\rightarrow\mathbb{R}^6$: 
\begin{equation}
\xi_v(t) = [\xi_{v_1}(t),\dots,\xi_{v_6}(t))]^T = P^{-1}_v(t)e_v(\xi_s(t),t),	\label{eq:ksi_v}
\end{equation}
where $P_v(t)=\text{diag}\{\left[\rho_{v_k}(t)\right]_{k\in\{1,\dots,6\}}\}$, and design the distributed control protocol  $u_i:\mathbb{R}^6\times\mathbb{R}^6\times[t_0,\infty)\rightarrow\mathbb{R}^6$:
\begin{equation}
u_i(\xi_{s},\xi_{v},t) = - c_i g_v J^{-T}_{\scriptscriptstyle O_i}(\overline{x})P^{-1}_v(t)R_v(\xi_v)\varepsilon_v(\xi_v), \label{eq:control_law_i}
\end{equation}
where $g_v > 0, J_{\scriptscriptstyle O_i}$ as defined in (\ref{eq:jacobian O_i}) and $c_i$ are predefined load sharing coefficients satisfying $\sum_{i\in\{1,\dots,N\}}c_i = 1$ and $0\leq c_i \leq 1, \forall i\in\{1,\dots,N\}$. Also, $R_v:\mathbb{R}^6\rightarrow\mathbb{R}^{6\times6}, \varepsilon_v:\mathbb{R}^6\rightarrow\mathbb{R}^6$ are defined as: 
\begin{eqnarray}
R_v(\xi_v) &=& \text{diag}\left\{\left[2(1 - \xi^2_{v_k})^{-1} \right]_{k\in\{1,\dots,6\}} \right\} \label{eq:r_v} \\
\varepsilon_v(\xi_v)  & = & \left[\ln\left(\dfrac{1+\xi_{v_1}}{1-\xi_{v_1}}\right),\dots,\ln\left(\dfrac{1+\xi_{v_6}}{1-\xi_{v_6}}\right) \right]^T \label{eq:epsilon_v}
\end{eqnarray}
The control law (\ref{eq:control_law_i}) can be written in vector form $\overline{u} = [u^T_1,\dots,u^T_N]^T$:
\small
\begin{equation}
\overline{u}(\xi_{s},\xi_{v},t) = U^{j'}_{j}(t) =  -C_g G^*(\overline{x})P^{-1}_v(t)R_v(\xi_v)\varepsilon_v(\xi_v), \label{eq:control_law_vector_form}
\end{equation}
\normalsize
where \small$C_g = g_v\text{diag}\{\left[c_iI_{6\times6}\right]_{i\in\{1,\dots,N\}}\}\in\mathbb{R}_{\geq 0}^{6N\times6N}$\normalsize, $G^*(\overline{x}) = [J^{-1}_{\scriptscriptstyle O_1},\dots,J^{-1}_{\scriptscriptstyle O_N}]^T\in\mathbb{R}^{6N\times6}$ and the notation $U^{j'}_{j}$ stands for the transition from $\pi_j$ to $\pi_{j'}$.	

\begin{rem}
It can be verified by (\ref{eq:x_dot_d}) and (\ref{eq:control_law_i}) that the proposed control protocol is distributed in the sense that each agent needs feedback only from the state of the object's center of mass, which can be obtained by (\ref{eq: object agent pos}) and (\ref{eq:object-end-effector jacobian}), as discussed in Section \ref{subsec:system_model}. The parameters needed for the computation of $\rho_{s_k},\rho_{v_k}, \eta_d$ as well as $g_{s_k},c_i, k\in\{1,\dots,6\},i\in\{1,\dots,N\}$ can be transmitted off-line to the agents. Moreover, the overall control scheme does not incorporate any prior knowledge of the model nonlinearities/disturbances or force/torque measurements at the contact points.
\end{rem}
\begin{rem}
Note that internal force regulation can be also guaranteed by including in \eqref{eq:control_law_vector_form} a term of the form $(I_{6N}-\tfrac{1}{N}G^*G^T)\hat{f}_{\text{int,d}}$, where $\hat{f}_{\text{int,d}}$ is a constant vector that can be transmitted off-line to the agents. Nevertheless, the computation of $G^*G^T$ requires knowledge of all grasping points $p_{\scriptscriptstyle E_i}$, which reduces to knowledge of the offsets $p^{\scriptscriptstyle O}_{\scriptscriptstyle E_i/O}$ (since the agents can compute $R_{\scriptscriptstyle O}$ and $p_{\scriptscriptstyle O}$), that can be also transmitted off-line to the agents.
\end{rem}

The next theorem summarizes the results of this section. 
\begin{thm}
Consider $N$ agents rigidly grasping an object with unknown coupling dynamics described by (\ref{eq:coupled dynamics 2}) and $\mathcal{A}(\overline{q}(t_0))\in\pi_j,j\in\{1,\dots,R\}$. Then, the distributed control protocol (\ref{eq:ksi_s})-(\ref{eq:epsilon_v}) guarantees that $\pi_j\rightarrow\pi_{j'}$ with time duration $\delta t_{j,j'}$ and all closed loop signals being bounded, and thus establishes a transition relation between $\pi_j$ and $\pi_{j'}$ for the coupled system object-agents, according to Def. \ref{def:transition}. 
\end{thm}

\begin{proof}
Differentiating (\ref{eq:ksi_s}) and (\ref{eq:ksi_v}) with respect
to time, employing (\ref{eq:e_s}), (\ref{eq:coupled dynamics 2}) as well as the facts that \small$\dot{x}_{\scriptscriptstyle O} = \dot{x}^*_d + P_v(t)\xi_v, \sum_{i=1}^N c_{i}=1$ \normalsize and substituting (\ref{eq:x_dot_d}), (\ref{eq:control_law_vector_form}), we obtain:
\small
\begin{eqnarray}
\dot{\xi}_{s_k}(\xi,t) & = & h_{s_k}(\xi,t) \nonumber \\
& = & -g_{s_k}\rho^{-1}_{s_k}(t)\varepsilon_{s_k}(\xi_{s_k}) - \rho^{-1}_{s_k}(t)(\dot{x}_{d_k}(t) \nonumber \\
&  &  + \dot{\rho}_{s_k}(t)\xi_{s_k} - \rho_{v_k}(t)\xi_{v_k}) \label{xi_s_dot}, \forall k\in\{1,\dots,6\} \label{eq:ksi_s_dot} \\
\dot{\xi}_v(\xi,t) & = & h_v(\xi,t) \nonumber \\
 & = & -g_vP^{-1}_v(t)\widetilde{M}^{-1}(\overline{x})P^{-1}_v(t)R_v(\xi_v)\varepsilon_v(\xi_v) \nonumber \\
& & - P^{-1}_v(t)\left[ \widetilde{M}^{-1}(\overline{x})( \widetilde{C}(\overline{x},\dot{\overline{x}})(P_v(t)\xi_v + \dot{x}_d^*) \right.    \nonumber \\
& &\left. + \widetilde{h}(\overline{x},\dot{\overline{x}}) + \widetilde{w}(\overline{x},t)) + \dot{P}_v(t)\xi_v + \dfrac{\partial}{\partial t}\dot{x}^*_d \right]  
\label{eq:ksi_v_dot}
\end{eqnarray}
\normalsize
for all $t\geq t_0$ with $\xi(t) = [\xi^T_s(t), \xi^T_v(t)]^T$. By defining $h:\mathbb{R}^6\times\mathbb{R}^6\times[t_0,\infty)\rightarrow \mathbb{R}^{12}$, we can write (\ref{eq:ksi_s_dot}), (\ref{eq:ksi_v_dot}) in compact form:
\begin{equation} 
\dot{\xi} = h(\xi,t) = [h^T_s(\xi,t),h^T_v(\xi,t)]^T \label{eq:ksi_dot}
\end{equation}
We further define the open and nonempty set $\Omega_{\xi} = \Omega_{\xi_s}\times\Omega_{\xi_v} \subset \mathbb{R}^{12}$ with $\Omega_{\xi_s} =\Omega_{\xi_v} = (-1,1)^6$.

Since $\mathcal{A}(\overline{q}(t_0))\in\pi_j$, Def. \ref{def:system in region} implies that $d(p_{\scriptscriptstyle O}(\overline{q}(t_0)), p^c_{\pi_j})<l_0$. Also, $p_d(t_0)=p^c_{\pi_j}$. Therefore, by choosing $\rho^{\scriptscriptstyle 0}_{s_k}=l_0, \forall k\in\{1,2,3\}$ as well as $\rho^{\scriptscriptstyle 0}_{s_k} > \lvert e_{s_k}(t_0)\rvert, \forall k\in\{4,5,6\}$ we guarantee that $\xi_s(t_0)\in\Omega_{\xi_s}$. Furthermore, by selecting $\rho^{\scriptscriptstyle 0}_{v_k} > \lvert e_{v_k}(t_0) \rvert, \forall k\in\{1,\dots,6\}$, we guarantee that $\xi_v(t_0)\in\Omega_{\xi_v}$. Thus, we conclude that $\xi(t_0)\in\Omega_\xi$. Additionally, $h$
is continuous on $t$ and locally Lipschitz on $\xi$ over $\Omega_\xi$. Therefore, according to Theorem $\ref{Th:dynamical systems}$, there exists a maximal solution $\xi(t)$ of (\ref{eq:ksi_dot}) on a time interval $[t_0,t_{\max})$ such that $\xi(t)\in\Omega_\xi,\forall t\in[t_0,t_{\max})$. Thus: 
\begin{equation}
\xi_{m_k}(t) = \dfrac{e_{m_k}}{\rho_{m_k}}\in(-1,1), \forall k\in\{1,\dots,6\}, m\in\{s,v\} \label{eq:ksi_m}
\end{equation}
$\forall t\in[t_0,t_{\max})$, from which we obtain that $e_{s_k}(t)$ and $e_{v_k}(t)$ are bounded by $\rho_{s_k}(t)$ and $\rho_{v_k}(t)$, respectively. Therefore, the error vectors $\varepsilon_{s_k}(\xi_s),\forall k\in\{1,\dots,6\}$ and $\varepsilon_v(\xi_v)$, as defined in (\ref{eq:epsilon_s}) and (\ref{eq:epsilon_v}), respectively, are well defined for all $t\in[t_0,t_{\max})$. Hence, consider the positive definite and radially unbounded functions $V_{s_k}:\mathbb{R}\rightarrow\mathbb{R}_{\geq 0}$ with $V_{s_k}(\varepsilon_{s_k}) = \varepsilon_{s_k}^2, \forall k\in\{1,\dots,6\}$. By differentiating $V_{s_k}$ with respect to time and substituting (\ref{eq:ksi_s_dot}), we obtain:
\small
\begin{equation}
\dot{V}_{s_k}  =  -\dfrac{ 4\varepsilon_{s_k}\rho^{-1}_{s_k}(t)}{1-\xi^2_{s_k}}(g_{s_k}\varepsilon_{s_k} + \dot{x}_{d_k}(t) + \dot{\rho}_{s_k}(t)\xi_{s_k} - \rho_{v_k}(t)\xi_{v_k} )  \nonumber 
\end{equation}
\normalsize
Next, since $\dot{x}_{d_k}, \rho_{v_k},\dot{\rho}_{s_k}$ are bounded by construction and $\xi_{s_k},\xi_{v_k}$ are bounded in $(-1,1)$ owing to (\ref{eq:ksi_m}), $\dot{V}_{s_k}$ becomes:
\begin{equation}
\dot{V}_{s_k} \leq \dfrac{4\rho^{-1}_{s_k}(t)}{1-\xi^2_{s_k}}(\overline{B}_s\lvert \varepsilon_{s_k} \rvert - g_{s_k}\lvert \varepsilon_{s_k}\rvert^2), \nonumber
\end{equation}
$\forall t\in[t_0,t_{\max})$, where $\overline{B}_s$ is an unknown positive constant independent of $t_{\max}$ satisfying \small $\overline{B}_s \geq \lvert \dot{x}_{d_k}(t) + \dot{\rho}_{s_k}(t)\xi_{s_k} - \rho_{v_k}(t)\xi_{v_k}  \rvert$ \normalsize. Therefore, we conclude that $\dot{V}_{s_k}$ is negative when \small $\lvert \varepsilon_{s_k} \rvert > \overline{B}_s g_{s_k}^{-1}$ \normalsize and subsequently that 
\begin{equation}
\lvert \varepsilon_{s_k}(\xi_{s_k}(t)) \rvert \leq \overline{\varepsilon}_{s_k} = \max\left\{\lvert \varepsilon_{s_k}(\xi_{s_k}(t_0))\rvert, \overline{B}_s g_{s_k}^{-1}\right\},  \label{eq:epsilon_s_bar}
\end{equation} 
$\forall t\in[t_0,t_{\max}), k\in\{1,\dots,6\}$. Furthermore, from (\ref{eq:epsilon_s}), taking the inverse logarithm, we obtain:
\small
\begin{equation}
-1 < \dfrac{e^{-\overline{\varepsilon}_{s_k}} - 1}{e^{-\overline{\varepsilon}_{s_k}} + 1} = \underline{\xi}_{s_k} \leq \xi_{s_k}(t) \leq \overline{\xi}_{s_k} = \dfrac{e^{\overline{\varepsilon}_{s_k}} - 1}{e^{\overline{\varepsilon}_{s_k}} + 1} < 1 \label{eq:epsilon_s_bars}
\end{equation}
\normalsize
$\forall t\in[t_0,t_{\max}), k\in\{1,\dots,6\}$. Due to (\ref{eq:epsilon_s_bar}), the reference velocity vector \small$\dot{x}^*_d(\xi_s,t)$\normalsize, as defined in (\ref{eq:x_dot_d}),  remains bounded for all $t\in[t_0,t_{\max})$. Moreover, invoking \small$\dot{x}_{\scriptscriptstyle O} = \dot{x}^*_d(\xi_s,t) + P_v(t)\xi_v$ \normalsize and (\ref{eq:ksi_m}), we also conclude the boundedness of $\dot{x}_{\scriptscriptstyle O}$ for all $t\in[t_0,t_{\max})$. Finally, differentiating $\dot{x}^*_d(\xi_s,t)$ with respect to time and employing (\ref{eq:ksi_s_dot}), (\ref{eq:ksi_m}) and (\ref{eq:epsilon_s_bars}), we conclude the boundedness of \small $\dfrac{\partial}{\partial t}\dot{x}^*_d(\xi_s,\dot{\xi}_s,t)$, \normalsize $\forall t\in[t_0,t_{\max})$.

Applying the aforementioned line of proof, we consider the positive definite and radially unbounded function $V_v:\mathbb{R}^6\rightarrow\mathbb{R}$ with $V_v(\varepsilon_v) =$ \small$ 0.5\lVert \varepsilon_v \rVert^2$\normalsize. By differentiating $V_v$ with respect to time, we substitute (\ref{eq:ksi_v_dot}) and by employing (i) the continuity of $\widetilde{M},\widetilde{C},\widetilde{h}$ and (ii) the boundedness of \small$w_{\scriptscriptstyle O},w_i,\xi_s,\xi_v,\dot{P}_v, \dfrac{\partial}{\partial t}\dot{x}^*_d$, \normalsize $\forall t\in[t_0,t_{\max})$, we obtain:
\small
\begin{equation}
\dot{V}_v \leq  \lVert P_v^{-1}(t)R_v(\xi_v)\varepsilon_v \rVert ( \overline{B}_v -  g_v\lambda_m \lVert P_v^{-1}(t)R_v(\xi_v)\varepsilon_v \rVert ) \nonumber
\end{equation}
\normalsize
$\forall t\in[t_0,t_{\max})$, where $\lambda_m$ is the minimum singular value of the positive definite matrix $\widetilde{M}^{-1}$ and $\overline{B}_v$ is a positive constant independent of $t_{\max}$, satisfying 
\small
\begin{eqnarray}
\overline{B}_v &\geq& \left\lVert \widetilde{M}^{-1}(\overline{x})(C(\overline{x},\dot{\overline{x}}) (P_v(t)\xi_v + \dot{x}_d^*(t)) + \widetilde{h}(\overline{x},\dot{\overline{x}}) + \right.  \nonumber \\
 & & \left.  \widetilde{w}(\overline{x},t) + \dot{P}_v\xi_v + \dfrac{\partial}{\partial t}\dot{x}_d^*(\xi_s,\dot{\xi}_s,t)  )  \right\rVert \nonumber
\end{eqnarray}
\normalsize
Therefore, $\dot{V}_v$ is negative when \small $\lVert P_v^{-1}(t)R_v(\xi_v)\varepsilon_v \rVert > \overline{B}_v(g_v\lambda_m)^{-1}$\normalsize, which, by employing the definitions of $P_v$ and $R_v$, becomes \small $\lVert \varepsilon_v \rVert > \overline{B}_v(g_v\lambda_m)^{-1} \overline{\rho}^{\scriptscriptstyle 0}_{v_k}$, \normalsize with $\overline{\rho}^{\scriptscriptstyle 0}_{v_k}=\max_{k\in\{1,\dots,6\}}\{\rho^{\scriptscriptstyle 0}_{v_k}\} $. Therefore, we conclude that 
\begin{equation}
\lVert \varepsilon_v(\xi_v(t)) \rVert \leq \overline{\varepsilon}_v = \max\left\{\varepsilon_v(\xi_v(t_0)), \overline{B}_v(g_v\lambda_m)^{-1}\overline{\rho}^{\scriptscriptstyle 0}_{v_k} \right\}, \nonumber
\end{equation}
$\forall t\in[t_0,t_{\max})$. Furthermore, from (\ref{eq:epsilon_v}), taking the inverse logarithm and invoking that $\lvert \varepsilon_{v_k} \rvert \leq \lVert \varepsilon_v \rVert$, we obtain:
\small
\begin{equation}
-1 < \dfrac{e^{-\overline{\varepsilon}_{v_k}} - 1}{e^{-\overline{\varepsilon}_{v_k}} + 1} = \underline{\xi}_{v_k} \leq \xi_{v_k}(t) \leq \overline{\xi}_{v_k} = \dfrac{e^{\overline{\varepsilon}_{v_k}} - 1}{e^{\overline{\varepsilon}_{v_k}} + 1} < 1 \label{eq:epsilon_v_bars}
\end{equation}
\normalsize
$\forall t\in[t_0,t_{\max}), k\in\{1,\dots,6\}$, which also leads to the boundedness of the distributed control protocol (\ref{eq:control_law_vector_form}).
Up to this point, what remains to be shown is that $t_{\max}$ can
be extended to $\infty.$ In this direction, notice by (\ref{eq:epsilon_s_bars})
and (\ref{eq:epsilon_v_bars}) that $\xi(t)\in\Omega'_\xi = \Omega'_{\xi_s}\times\Omega'_{\xi_v},\forall t\in[t_0,t_{\max})$, where: 
\begin{equation}
 \Omega'_{\xi_m} = [\underline{\xi}_{m_1},\overline{\xi}_{m_1}]\times\cdots\times[\underline{\xi}_{m_6},\overline{\xi}_{m_6}], m\in\{s,v\} \nonumber  
\end{equation}
are nonempty and compact subsets of $\Omega_{\xi_s}$ and $\Omega_{\xi_v}$, respectively. Hence, assuming that $t_{\max} < \infty$ and since $\Omega'_{\xi}\subset\Omega_{\xi}$, Proposition \ref{Prop:dynamical systems} dictates the existence of a time instant $t'\in[t_0,t_{\max})$ such that $\xi(t')\notin\Omega'_\xi$, which is a clear contradiction. Therefore, $t_{\max} = \infty$. Thus, all closed loop signals remain bounded and moreover $\xi(t)\in\Omega'_\xi,\forall t\geq t_0$. 

By multiplying (\ref{eq:epsilon_s_bars}) with $\rho_{s_k}(t)$, we obtain $\lvert e_{s_k}(t) \rvert < \rho_{s_k}(t), \forall k\in\{1,\dots,6\}$ and hence $\lvert e_{s_k}(t) \rvert < l_0,\forall k\in\{1,2,3\}, t\in[t_0,\infty)$, since $\rho^{\scriptscriptstyle 0}_{s_k}=l_0, \forall k\in\{1,2,3\}$. Therefore, $p_{\scriptscriptstyle O}(\overline{q}(t))\in\mathcal{B}(p_d(t),l_0), \forall t\geq t_0$ and, consequently, $p_{\scriptscriptstyle O}(\overline{q}(t_0+\delta t_{j,j'}))\in\mathcal{B}(p^c_{\pi_{j'}},l_0)$, since $p_d(t_0+\delta t_{j,j'}) = p^c_{\pi_{j'}}$. Also, since $p_{\scriptscriptstyle O}(\overline{q}(t))\in\mathcal{B}(p_d(t),l_0)$, we deduce that $\mathcal{B}(p_{\scriptscriptstyle O}(\overline{q}(t)),\hat{L})\in\mathcal{B}(p_d(t),l_0+\hat{L})$ and invoking (\ref{eq:p_s in L_hat}) and (\ref{eq:desired_tr}), we conclude that $p_s\in\pi_j\cup\pi_{j'}, \forall t\in[t_0,t_0+\delta t_{j,j'}]\subset[t_0,\infty)$, and therefore a transition relation with time duration $\delta t_{j,j'}$ is successfully established. 
 \end{proof}
 
 \begin{rem}
 Instead of employing the control protocol (\ref{eq:control_law_vector_form}) over $[t_0,\infty)$, we can define it over a finite time interval as $\overline{u}=U^{j'}_{j}([t_0,t_0+\delta t_{j,j'}))$. In that case, it follows by the continuity of $d, p_{\scriptscriptstyle O}, p_d$ that $\lim_{t\rightarrow (t_0+\delta t_{j,j'})^{-}} d(p_{\scriptscriptstyle O}(t),p_d(t)) = d(p_{\scriptscriptstyle O}(t_0+\delta t_{j,j'}),p^c_{\pi_{j'}})$ and therefore, the transition $\pi_j\rightarrow\pi_{j'}$ with time duration $\delta t_{j,j'}$ is still achieved. Moreover, the predefined selection of $\delta t_{j,j'}$ for each transition $\pi_j \rightarrow \pi_{j'}$ is related to the control capabilities of the agents, since smaller $\delta t{j,j'}$ will produce larger, but still bounded, $\dot{x}^*_d$ and $\overline{u}$.
 \end{rem}

\subsection{High-Level Plan Generation} \label{subsec:High level}
The second part of our solution is the derivation of a high-level plan that satisfies the given MITL formula $\phi$ and can be generated using standard techniques from automata-based formal verification methodologies. Thanks to our proposed control law that allows the transition $\pi_j\rightarrow\pi_{j'}$ for all $\pi_j\in\Pi$ with $\pi_{j'}\in\mathcal{D}(\pi_j)$ in a predefined time interval $\delta t_{j,j'}$, we can abstract the motion of the coupled system object-agents as a finite weighted transition system (WTS) \cite{baier2008principles} $\mathcal{T} = \{\Pi, \Pi_0, \rightarrow, \mathcal{AP}, \mathcal{L},  \gamma \}$, where $\Pi$ is the set of states defined in Section \ref{subsec:wsp discret}, $\Pi_0\subset\Pi$ is a set of initial states, $\rightarrow\subseteq\Pi\times\Pi$ is a transition relation according to Def. \ref{def:transition}, $\mathcal{AP}$ and $\mathcal{L}$ are the atomic propositions and the labeling function, respectively, as defined in Section \ref{subsec:specf}, and $\gamma:\rightarrow\rightarrow\mathbb{R}_{\geq 0}$ is a map that assigns to each transition its time duration, i.e., $\gamma(\pi_j\rightarrow\pi_{j'}) = \delta t_{j,j'}$. 
Therefore, by designing the switching control protocol $U^{r_{j+1}}_{r_j}(t)$ from (\ref{eq:control_law_vector_form}):
\begin{equation}
\overline{u} = U^{r_{j+1}}_{r_j}(t), \forall t\in[t_j, t_j+\delta t_{r_j,r_{j+1}}), j\in\mathbb{N}, \label{eq:switch_control}
\end{equation}
with (i) $t_1 = 0$, (ii) $t_{j+1} = t_{j} + \delta t_{r_j,r_{j+1}}$ and (iii) $r_j\in\{1,\dots,R\}, \forall j\in\mathbb{N}$, 
we can define a \textit{timed run} of the WTS as an infinite sequence $r = (\pi_{r_1},t_1)(\pi_{r_2},t_2)\dots$, where $\pi_{r_1}\in\Pi_0$ with $\mathcal{A}(\overline{q}(0))\in\pi_{r_1}, \pi_{r_j}\in\Pi, r_j\in\{1,\dots,R\}$ and $t_j$ are the corresponding time stamps such that $\mathcal{A}(\overline{q}(t_j))\in\pi_{r_j}, \forall j\in\mathbb{N}$. Every timed run $r$ generates a \textit{timed word} $w(r) = (\mathcal{L}(\pi_{r_1}),t_1)(\mathcal{L}(\pi_{r_2}),t_2)\dots$ over $\mathcal{AP}$ where $\mathcal{L}(\pi_{r_j}), j\in\mathbb{N}$ is the subset of the  atomic propositions $\mathcal{AP}$ that are true when $\mathcal{A}(\overline{q}(t_j))\in\pi_{r_j}$.

 The given MITL formula $\phi$ is translated into a \textit{Timed Büchi Automaton} $\mathcal{A}^t_{\phi}$ \cite{Alur94_timed_automata} and the product $\mathcal{A}_{p}=\mathcal{T}\otimes\mathcal{A}^t_{\phi}$ is built \cite{baier2008principles}. The projection of the accepting runs of $\mathcal{A}_p$ onto $\mathcal{T}$ provides a \textit{timed run} $r_\phi$ of $\mathcal{T}$ that satisfies $\phi$; $r_\phi$ has the form $r_{\phi} = (\pi_{r_1},t_1)(\pi_{r_2},t_2)\dots$, i.e., an infinite\footnote{It can be proven that if such a run exists, then there also exists a run that can be always represented as a finite prefix followed by infinite repetitions of a finite suffix \cite{baier2008principles}.} sequence of regions $\pi_{r_j}$ to be visited at specific time instants $t_j$ (i.e., $\mathcal{A}(\overline{q}(t_j))\in\pi_{r_j}$) with $t_1 = 0$ and $t_{j+1} = t_j + \delta t_{r_j,r_{j+1}}, r_j\in\{1,\dots,R\}, \forall j\in\mathbb{N}$. More details on the technique are beyond the scope of this paper and the reader is referred to \cite{baier2008principles, Alex16,Alur94_timed_automata}. 	
 
 The execution of $r_\phi=(\pi_{r_1},t_1)(\pi_{r_2},t_2)\dots$ produces a trajectory $\overline{q}(t), t\in\mathbb{R}_{\geq 0}$, with timed sequence $\beta_\phi = (\overline{q}_1(t),t_1)(\overline{q}_2(t),t_2)\dots$, with $\mathcal{A}(\overline{q}_j(t_j))\in\pi_{r_j}, \forall j\in\mathbb{N}$. Following Def. \ref{def:specification}, $\beta_{\phi}$ has the timed behavior $\sigma_{\beta_\phi} = (\sigma_1,t_1)(\sigma_2,t_2)\dots$ with $\sigma_j\in\mathcal{L}(\pi_{r_j})$, for $\mathcal{A}(\overline{q}_j(t_j))\in\pi_{r_j}, \forall j\in\mathbb{N}$. Since all $\pi_{r_j}$ belong to $r_\phi, \forall j\in\mathbb{N}$, the latter implies that $\sigma_{\beta_\phi} \models \phi$ and therefore that $\beta_\phi$ satisfies $\phi$. The aforementioned discussion is summarized as follows:
\begin{thm}
The execution of $r_\phi=(\pi_{r_1},t_1)(\pi_{r_2},t_2)\dots$ of $\mathcal{T}$ that satisfies $\phi$ guarantees a timed behavior $\sigma_{\beta_\phi}$ of the coupled system object-agents that yields the satisfaction of $\phi$ and provides, therefore, a solution to Problem \ref{problem 1}. 
\end{thm}

\begin{figure}[!btp]
 \centering
  \subfloat[A robotic agent]{\includegraphics[trim = 0cm -0.5cm 0cm 2	cm,width=0.17\textwidth, height=0.13\textheight]{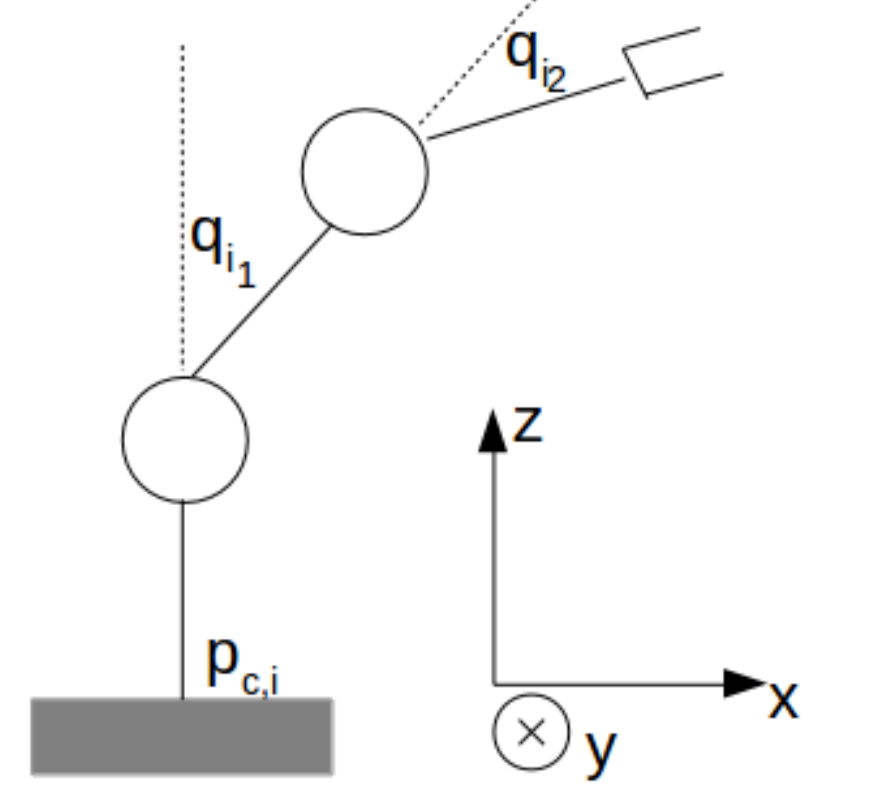}\label{fig:sim_agents}}
  \subfloat[Top view of the initial workspace]{\includegraphics[trim = 0cm 0cm 0cm 0.75cm,width=0.3\textwidth, height=0.18\textheight]{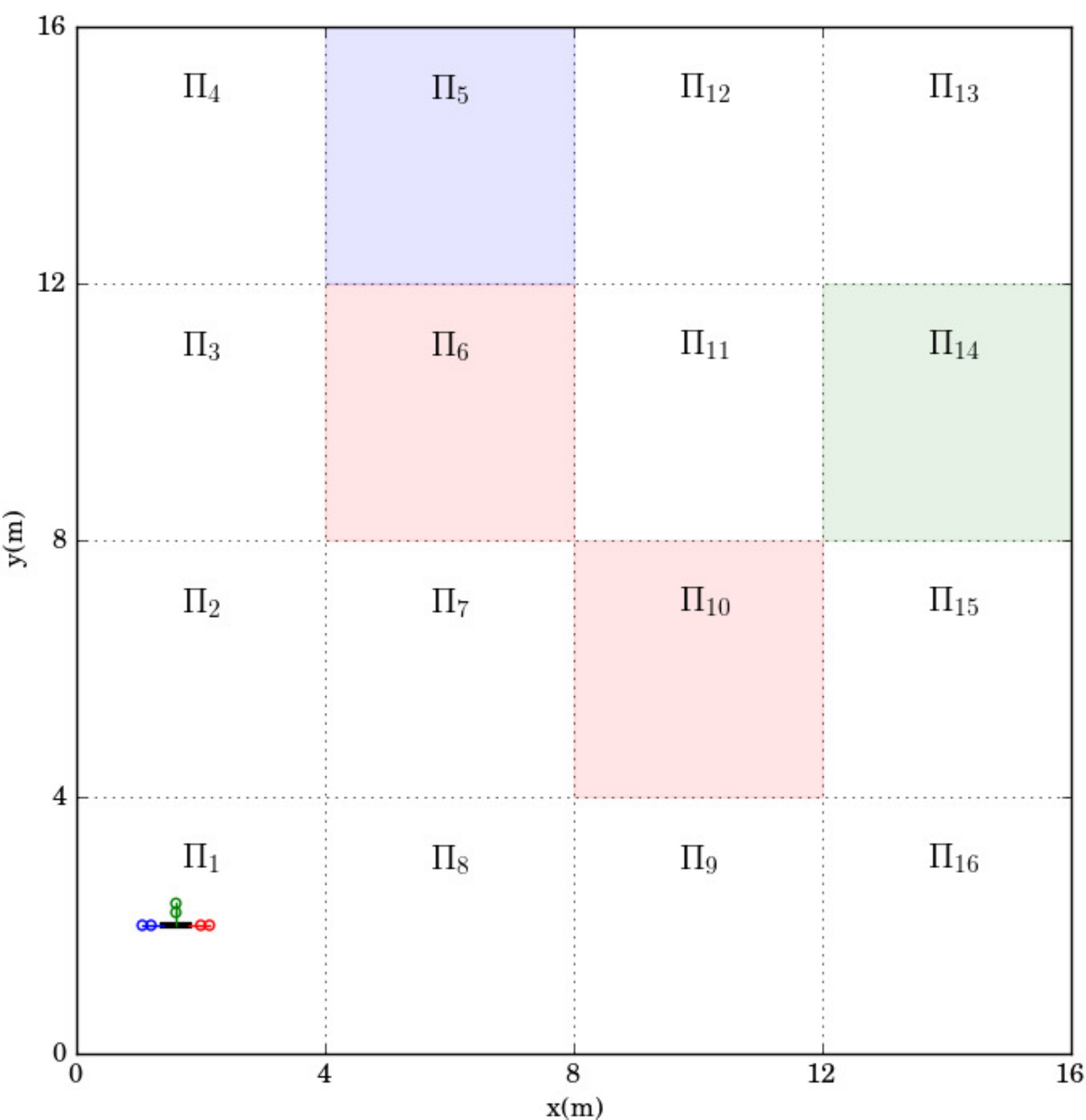}\label{fig:simulation workspace}}
  \caption{(a): A robotic agent. (b): Top view of the initial workspace with three agents (with blue, red and green) carrying an object (with black) in $\pi_1$. Goal regions are marked with blue and green whereas obstacle regions with red.}
  \label{fig:sim_initial}
\end{figure}

\section{Simulation Results}\label{sec:Simulation Results}
The validity of the proposed framework is verified through simulation studies. We consider three robots consisting of two joints, three links and a mobile platform that is able to move along the $x-y$ plane (Fig. \ref{fig:sim_initial}\subref{fig:sim_agents}). We denote as $q_{i}=[p^{T}_{c,i},\eta^{T}_{c,i}, q_{i_{1}},  q_{i_{2}}]^{T}\in\mathbb{R}^{8}$ the generalized coordinates of each agent, where $p_{c,i},\eta_{c,i}\in\mathbb{R}^3$ denote the position and orientation of the $i$th agent's platform and $q_{i_1},q_{i_2}$ are the joint angles, $\forall i\in\{1,2,3\}$. The agents grasp a rigid rod of length $0.4\text{m}$. The initial coordinates of the system are taken as $p_{c,1}(0)=[1.06, 2, 0.1]^{T}\text{m},p_{c,2}(0)=[2.14, 2, 0.1]^{T}\text{m}$, $p_{c,3}(0)=[1.6, 2.34, 0.1]^{T}\text{m}$, $\eta_{c,1}(0)=[0, 0, 0]^{T}\text{r}$, $\eta_{c,2}(0)=[0, 0, \pi]^{T}\text{r}$, $\eta_{c,3}(0)=[0, 0, -\pi/2]^{T}\text{r}$, $q_{i_{1}}=q_{i_{2}}=\frac{\pi}{4}\text{r}$, $\forall i\in\{1,2,3\}$, and $p_{\scriptscriptstyle O}(0) = [1.6,2,0.44]^T\text{m}$, $\eta_{\scriptscriptstyle O}(0)=[0,0,0]^T\text{r}$. 
The workspace is partitioned into $R=16$ regions, with $\hat{L}=1.5\text{m}, l_0=0.5\text{m}$ and $p^c_{\pi_1}=[2,2,2]^T\text{m}$, from which it can be verified that $\mathcal{A}(\overline{q}(0))\in\pi_1$. 
We further choose $\delta t_{j,j'}=5$s, $\forall j,j'\in\{1,\dots,16\}$ and we define the atomic propositions  $\mathcal{AP}=\{\mathsf{``blue"},\mathsf{``green"},\mathsf{``obs"},\emptyset\}$ representing goal ($\mathsf{``blue"}$ and $\mathsf{``green"}$) and obstacle ($\mathsf{``obs"}$) regions as well as $\mathcal{L}(\mathsf{\pi}_{5})=\{\mathsf{``blue"}\},\mathcal{L}(\mathsf{\pi}_{14})=\{\mathsf{``green"}\},\mathcal{L}(\mathsf{\pi}_{6})=\mathcal{L}(\mathsf{\pi}_{10})=\{\mathsf{``obs"}\}$ and $\mathcal{L}(\mathsf{\pi}_{j})=\emptyset$ for the remaining regions. Fig. \ref{fig:sim_initial}\subref{fig:simulation workspace} depicts the aforementioned workspace.

We consider the MITL formula $\phi=(\square_{[0,\infty)}\neg\mathsf{``obs"})\land\lozenge_{[0,50]}(\mathsf{``green"}\land\lozenge_{[0,20]}\mathsf{``blue"})$, which represents the behavior of (i) always avoiding the obstacle regions and (ii) eventually within $50$s visiting the $\mathsf{``green"}$ region and after $20$s the $\mathsf{``blue"}$ region. By following the procedure described in Section \ref{subsec:High level}, we obtain the accepting timed run $r = (\pi_{r_1},t_1)((\pi_{r_2},t_2))\dots=(\pi_1,0) (\pi_2,5) (\pi_3,10) (\pi_4,15) (\pi_5,20) (\pi_{12},25) (\pi_{13},30)\\(\pi_{14},35) (\pi_{11},40) (\pi_{12},45) (\pi_5,50)$.

Regarding each transition $\pi_{r_j}\rightarrow\pi_{r_{j+1}},j\in\{1,\dots,10\}$, we chose $\eta_d=[0,0,0]^T, \rho^{\scriptscriptstyle 0}_{s_k}=l_{0}=0.5,\rho^{\scriptscriptstyle \infty}_{s_k}=0.01$ and $l_{s_{k}}=1$ as well as $\rho^{\scriptscriptstyle 0}_{v_k}=2\lvert e_{v_{k}}(t_j)\rvert+0.1,\rho^{\scriptscriptstyle \infty}_{v_{k}}=0.01$ and $l_{v_{k}}=1, \forall k\in\{1,\dots,6\}$. The control gains were chosen as $g_{s_k}=0.1$ and $g_v=2.5, \forall k\in\{1,\dots,6\}$ to retain the required input signals $u_{i}$ within feasible ranges that can be implemented by real actuators. Finally, $w_{\scriptscriptstyle O},w_{i}$ and $f_{i}$ were chosen as sinusoidal functions of time and the load sharing coefficients were selected as $c_{1}=0.5,c_{2}=0.35$ and $c_{3}=0.15$ to demonstrate a potential difference in the power capabilities of the agents.
Fig. \ref{fig:sim_res} depicts the transitions of the coupled system object-agents. It can be deduced from the figure that $p_{\scriptscriptstyle O}\in\mathcal{B}(p_d(t),l_{0}), \forall t\in[0,50]$ and therefore $p_s\in\pi_{r_j}\cup\pi_{r_{j+1}}, \forall p_s\in\mathcal{S}_{\overline{q}},j\in\{1,\dots,10\}$, verifying thus the theoretical results. Moreover, Fig. \ref{fig:control_inputs} shows the control inputs of the three agents, demonstrating the effect of the load sharing coefficients. 

\begin{figure}[!btp]
\centering
\includegraphics[trim = 0cm 1cm 0cm -1cm,width=0.4\textwidth]{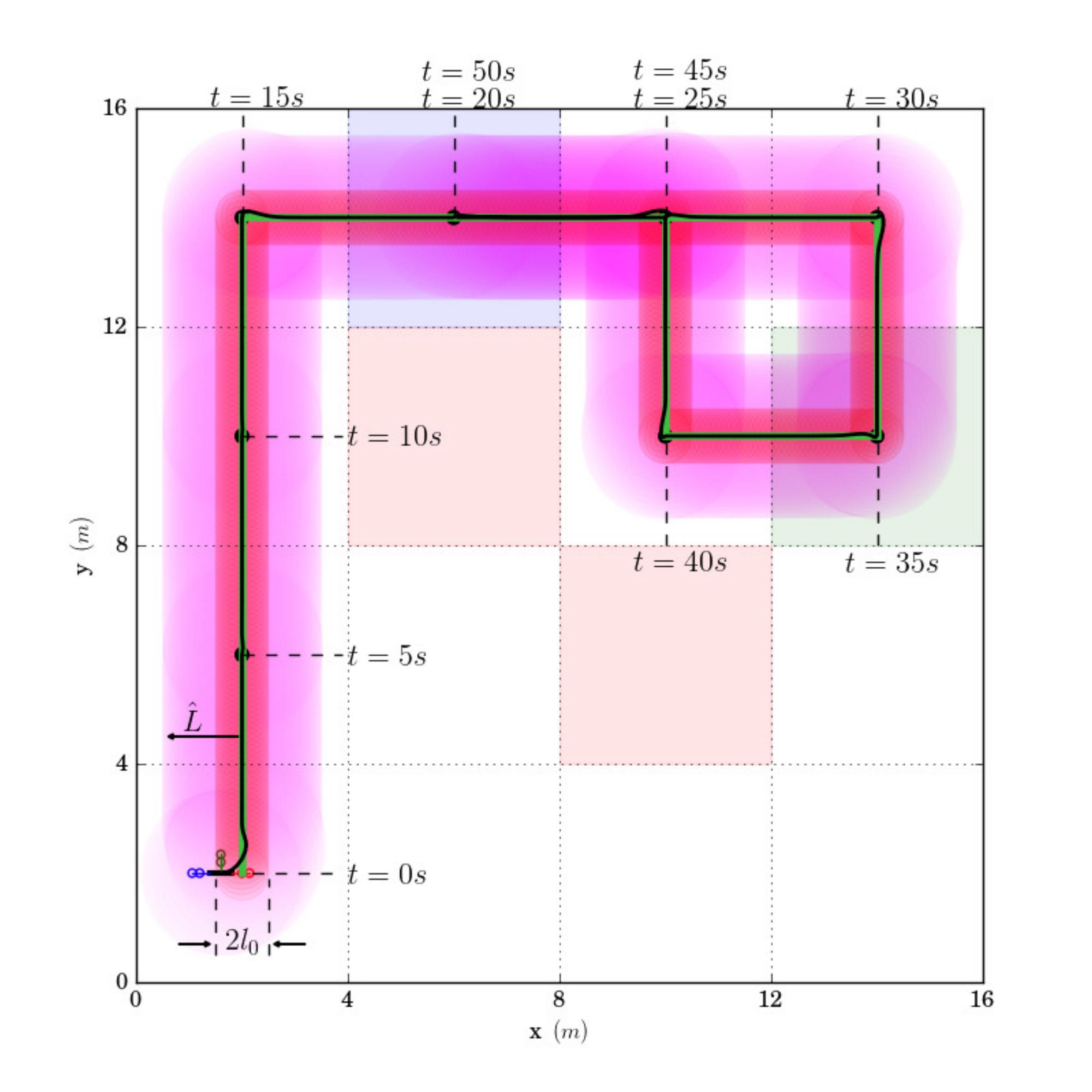}

\caption{The desired object trajectories (with green), the actual object trajectories (with black), the domain specified by $\mathcal{B}(p_d(t),l_{0})$ (with red) and the domain specified by $\mathcal{B}(p_{\scriptscriptstyle O}(t),\hat{L})$, (with purple) for $t\in\left[0,50\right]$ sec. Since $p_{\scriptscriptstyle O}(t)\in\mathcal{B}(p_d(t),l_0),\forall t\in[0,50]$s, the desired timed run is successfully executed. \label{fig:sim_res}}
\end{figure} 

\begin{figure}[!btp]
\centering
\includegraphics[trim = 0cm 0cm 0cm 0cm,width=0.44\textwidth,height=0.4\textwidth]{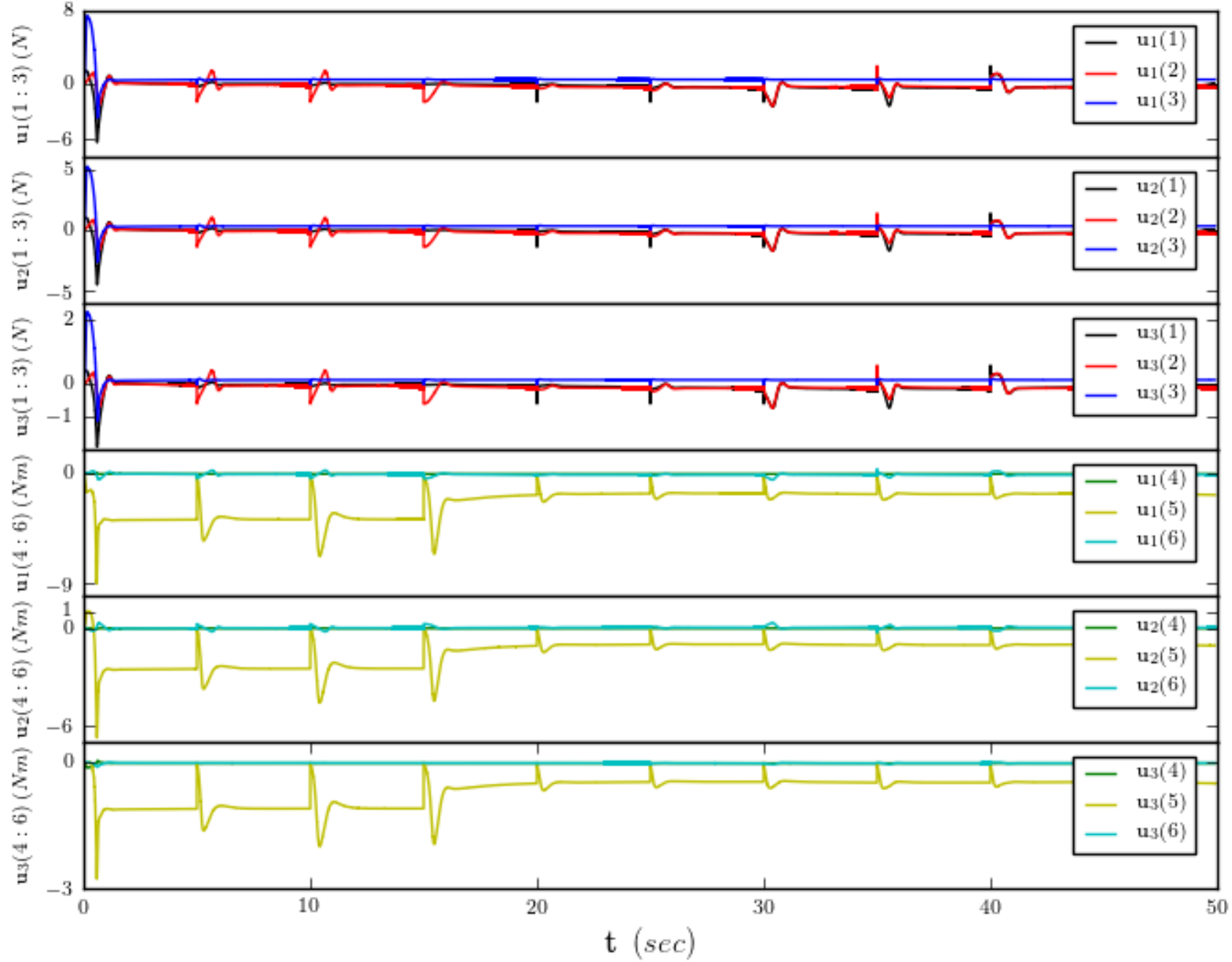}
\caption{The resulting control inputs $u_i = [u_i(1),\dots,u_i(6)]^T, i\in\{1,2,3\}$. Top three plots: $u_i(1),u_i(2),u_i(3)$, for $i=1,2,3$, respectively. Bottom three plots: $u_i(4),u_i(5),u_i(6)$, for $i=1,2,3$, respectively.}
\label{fig:control_inputs} 
\end{figure}

\section{Conclusion and Future Work} \label{sec:Conclusion}
In this work we proposed a hybrid control strategy for the cooperative manipulation of an object by $N$ agents under MITL specifications. In particular, we developed a robust decentralized control protocol that allowed us to abstract the motion of the coupled system object-agents as a finite transition system. Then, we employed standard formal-verification tools for the derivation of a path that satisfied the high level goal. Future efforts will be devoted towards considering non-rigid grasps and compensating for uncertain geometric characteristics of the objects.

\bibliographystyle{IEEEtran}
\bibliography{bib}
\end{document}